%% file: main_sgd_v2.tex
\documentclass{amsart}

\input{dependencies/dependencies}

\usepackage{soul}
\usepackage{subcaption}

\title{Complex Stochastic Gradient Descent  and Directional Bias in Reproducing Kernel Hilbert Spaces}
\date{\today}

\author[N. Alpay]{Natanael Alpay}
\address{(NA)
Department of Mathematics\\
University of California, Irvine,
Irvine, CA 92697 \\
USA}
\email{nalpay@uci.edu}

\author[E. Battaglia]{Emeric Battaglia}
\address{(EB)
Department of Mathematics\\
University of California, Irvine,
Irvine, CA 92697 \\
USA}
\email{ebattagl@uci.edu}


\begin{document}

\begin{abstract}
    Stochastic Gradient Descent (SGD) is a known stochastic iterative method popular for large-scale convex optimization problems due to its simple implementation and scalability. Some objectives, such as those found in complex-valued neural networks, benefit from updates like in SGD and Gradient Descent (GD) with a newly defined ``gradient'' that allows for complex parameters. This complex variant of the SGD/GD methods has already been proposed, but convergence guarantees without analyticity constraints have not yet been provided. We propose a variant of SGD (complex SGD) that allows for complex parameters, and we provide convergence guarantees under assumptions that parallel those from the real setting. Notably, these results extend to GD as well, and with the same set of assumptions, we confirm that some directional bias results extend from the real to the complex setting for kernel regression problems. We provide empirical results demonstrating the efficacy of the complex SGD in kernel regression problems utilizing complex reproducing kernel Hilbert spaces. In particular, we demonstrate we may recover superoscillation functions and Blaschke products from the Fock Space and Hardy Space, respectively, as the optimal functions for a particular choice of a loss function.
\end{abstract} 

\maketitle

\section{Introduction}
    Stochastic iterative methods, such as Stochastic Gradient Descent (SGD), are pivotal in minimizing objective or loss functions over large datasets. Much of the literature surrounding such minimization techniques assumes the system's parameters are real-valued. However, some models, such as complex-valued neural networks, are more appropriately described with complex-valued parameters. In such instances, it is possible to translate the parameters from $\C^{n}$ to $\R^{2n}$ and perform classical SGD. However, it has been noted this has less-than-desirable consequences such as an increased parameter count and a worse learning convergence properties \cite{complex_GD}; indeed, using an alternative definition of a gradient such as the  Wirtinger gradient (which allows for complex inputs) has proven to be more useful  \cite{brandwood1983complex,kreutzdelgado2009complex}, and has been studied in \cite{widrow1975complex,brandwood1983complex,vandenBos1994complex,hjorungnes2007complex,koor2023wirtinger,xu2015convergence,zhang2014convergence}. In particular, the authors of \cite{xu2015convergence,zhang2014convergence} have provided monotonicity results and some convergence results, but assumptions about analyticity are made. We later show that using the same principles from Wirtinger calculus yields a gradient definition that, under assumptions almost identical to the real setting, have provable convergence guarantees that make no assumptions about analyticity.
    
    In the real-setting for linear regression problems, there has been work demonstrating that SGD, depending on the choice of step-size, experiences directional bias \cite{dir_bias_sgd, luo2022directional}. To the best of the authors' knowledge, these results have not been extended to the complex setting. Because of the connection to another known algorithm, the random Kaczmarz method (RK), we elect to show that for overdetermined systems, a small step-size choice yields convergence primarily in the direction associated with the smallest singular value similarly to the real setting.
    
    In particular, the results of SGD are naturally applicable to kernel regression. For real reproducing kernel Hilbert spaces (RKHS's), the celebrated representer theorem, given by Kimeldorf and Wahba in 1971, reduces the problem of minimizing a regularized least-squares loss over an infinite-dimensional Hilbert space to a finite-dimensional space of coefficient vectors \cite{kimeldorf1971tchebycheffian}. However, a recent work \cite{AlpayDeMartinoDiki2025ComplexRepresenter} has shown this theorem may be extended to complex RKHS's, in which case classical SGD may no longer be employed. This can be remedied by using a complex version of SGD. With this tool, we may recover some examples of optimal functions for complex regularized least-squares problems derived analytically in \cite{AlpayDeMartinoDiki2025ComplexRepresenter}. One such function is the superoscillation sequence in the Fock space, which play an important role in quantum mechanics, signal analysis, and superresolution imaging \cite{JordanHowellVamivakasKarimi2025Superoscillations,ChenWenQiu2019Superoscillation}. In the radial basis function (RBF) space, the optimal function is the RBF supershift of the first kind. The associated kernel is commonly found in approximation theory and widely used in machine learning \cite{broomhead1988radial,hofmann2008kernel}. Finally, Blaschke product, which are the optimal function in the Hardy space, are fundamental objects in complex analysis and play an important role in operator theory, dynamical systems, and signal processing \cite{FL,GMW,mashreghi2013blaschke}. 

\section{Notation}
    \label{sec:notation}
    Complex vectors will be expressed as $\z=(z_1,\dots, z_n)^T\in \C^n$, and the vector whose entries are the complex conjugate of the entries of $\z$, i.e. $(\cl{z_1},\dots, \cl{z_n})^T$, will be denoted by $\cl{\z}$. For a matrix $A\in \C^{m\times n}$, $A\dl$ will denote the conjugate transpose of $A$, as usual.
    
    The standard complex inner product of $\C^n$ over $\C$ will be denoted by
    \[
    \lll \z,\w\rrr \Def \sum_{i=1}^n z_i\cl{w_i},
    \]
    where $\z,\w\in\C^n$. The modified inner product of the vector space $\C^n$ over $\R$ will be defined as
    \[
    \lll \z,\w\rrr_\R \Def \Re (\lll \z,\w\rrr),
    \]
    for $\z,\w\in\C^n$.
    
    The Wirtinger derivative of $f(z)$ where $z=x+iy$ for $x,y\in \R$ will be denoted as
    \[
    \del_{\cl{z}}f=\frac{\del f}{\del \cl{z}}=\frac{1}{2}\groupp{\frac{\del f}{\del x}+i\frac{\del f}{\del y}},
    \]
    when $f$ is not analytic, and 
    \[
        \del_{z}f=\frac{\del f}{\del z}=\frac{1}{2}\groupp{\frac{\del f}{\del x}-i\frac{\del f}{\del y}},
    \]
    when $f$ is analytic. For $\z\in \C^n$, we will write
    \[
    \nabla_{\cl{\z}}f=(\del_{\cl{z_1}}f,\dots, \del_{\cl{z_n}}f)^T.
    \]
    to mean the Wirtinger gradient of $f$ when $f$ is not analytic. Otherwise, we will use 
    \[
        \nabla_\z f=(\del_{z_1}f,\dots, \del_{z_n}f)^T
    \]
    The term $\nabla f$ will assume its normal definition in the event that the domain of $f$ is real-valued, and when the domain is complex-valued, we will define $\nabla f \Def 2\nabla_{\cl{\z}}f$ for non-analytic $f$ and $\nabla f\Def 2\nabla_\z f$ for analytic $f$. We will call $\nabla f$ the gradient of $f$ in either case when no confusion may arise. If distinction between the cases is required, we will call the former the real gradient and the latter the complex gradient.

\section{Background}

    \subsection{Classical SGD}
    \label{sec:classical_sgd_background}
        In the classical implementation of SGD, the goal is to minimize an objective function $F:\R^n\ra \R$ of the form
        \[
            F(\x)=\E_{\omega\sim \mcal D} f(\x;\omega)
        \]
        by iteratively updating the $t$-th iterate in the negative direction of the sampled gradient $\nabla f(\x^{(t)};\omega)$, i.e.
        \begin{equation}
            \x^{(t+1)} = \x^{(t)} - \eta_t \nabla f(\x^{(t)};\omega), \label{eq:sgd_update}
        \end{equation}
        where $\eta_t$ is the $t$-th step size, or learning rate, and $\omega$ is sampled independently from a distribution $\mcal D$.
        
        There are many convergence results for SGD showing either $\x^{(k)}\ra \x_\ast$ or $F(\x^{(k)})\ra F(\x_\ast)$ where $\x_\ast\in \argmin F$. Because our work is largely focused on showing SGD may be extended to loss or objective functions accepting complex inputs under the right conditions, we have chosen three elementary classical SGD results presented in \cite{handbook} to reproduce in the complex setting. The conclusions of the following three theorems are identical to the complex version later presented in \Cref{sec:complex_results}, but the hypotheses are slightly different. Of course, more refined results, such as those found in \cite{moulines2011non} and \cite{needell2014stochastic} are easily adaptable using the same (or a very similar) set of assumptions presented in \Cref{sec:problem_setup}.
        
        \begin{theorem}[Convergence in polynomial time (Classical Setting \cite{handbook})]
        \label{thm:avg_iter_converge} 
            Suppose $F(\x)=\E_{\omega\sim \mcal D}f(\x;\omega)$ is bounded below and for each $\omega$, $f(\x;\omega)$ is convex $\x$ and $\nabla f(\x;\omega)$ is $L$-Lipschitz. If $\E\|\nabla f(\x_\ast;\omega)\|^2<\infty$, $T\geq 1$, and $\eta_t < \frac{1}{4L}$ for all $t\leq   T$, then
            \[
                \E\groupb{F\groupp{\sum_{t=0}^{T-1}\frac{\eta_t}{\sum_{i=0}^{T-1}\eta_i}\x^{(t)}}-F(\x_\ast)} \leq \frac{\|\x^{(0)}-\x_\ast\|^2}{\sum_{t=0}^{T-1} \eta_t}+\frac{2\s_\ast \sum_{t=0}^{T-1}\eta_t^2}{\sum_{t=0}^{T-1} \eta_t},
            \]
            where $\x^{(t)}$ is the $t$-th classical SGD iterate generated by \Cref{eq:sgd_update} and $\s_\ast\Def \E\|\nabla f(\x_\ast;\omega)\|^2<\infty$.
        \end{theorem}
        
        \begin{theorem}[Exponential Convergence (Classical Setting \cite{handbook})]
        \label{thm:mu_convex_convergence}
            Under the same assumptions as in \Cref{thm:avg_iter_converge} and the additional assumption that $F$ is $\mu$-strongly convex, the sequence of iterations $\{\x^{(t)}\}_{t=0}^\infty$ generated by \Cref{eq:sgd_update} with step sizes $\eta_t$ satisfying $0<c\leq   \eta_t<\frac{1}{2L}$ for all $t$ satisfies
            \[
                \E\|\x^{(t+1)}-\x_\ast\|^2\leq (1-c\mu)^{t+1}\|\z^{(0)}-\x_\ast\|^2+\frac{\s_\ast}{2L^2c\mu},
            \]
            where $\s_\ast \Def \E\|\nabla f(\x_\ast;\omega)\|^2<\infty$. In the case that $\eta_t$ is constant, i.e. $\eta_t= \eta$ for all $t$, we may instead write
            \[
                \E\|\x^{(t+1)}-\x_\ast\|^2\leq (1-\eta \mu)^{t+1}\|\x^{(0)}-\x_\ast\|^2+\frac{2\s_\ast \eta}{\mu}.
            \]
        \end{theorem}
        
        \begin{theorem}[Stationary Convergence of SGD with Adaptive Step Sizes (Classical Setting \cite{handbook})]
        \label{thm:stationary_adaptive}
            Suppose the same assumptions as in \Cref{thm:avg_iter_converge} hold along with $\E\|\nabla f(\x;\omega)-\nabla F(\x)\|^2\leq   \s^2$ for some $\s>0$.
            If $\{\x^{(t)}\}_{t\geq 0}$ is the sequence of generated by \Cref{eq:sgd_update} with step sizes satisfying
            \[
                0<\eta_t\leq \frac{1}{L}\qquad (t\geq 0),
            \]
            then for every $T\geq 1$,
            \[
                \sum_{t=0}^{T-1}\eta_t\,\E\|\nabla F(\x^{(t)})\|^2 \leq 2\big(F(\x^{(0)})-F_\star\big) + L\sigma^2\sum_{t=0}^{T-1}\eta_t^2,
            \]
            where $F_\star:=\inf_{\z\in\C^n}F(\z)>-\infty$.
            Consequently,
            \[
                \min_{0\leq   t\leq   T-1}\E\|\nabla F(\x^{(t)})\|^2 \leq \frac{2(F(\x^{(0)})-F_\star)+L\sigma^2\sum_{t=0}^{T-1}\eta_t^2}{\sum_{t=0}^{T-1}\eta_t}.
            \]
            In particular, if $\eta_t=\frac{1}{L\sqrt{T}}$ for $t=0,\dots,T-1$, then
            \[
                \min_{0\leq t\leq T-1}\E\|\nabla F(\x^{(t)})\|^2\leq \frac{2L(F(\x^{(0)})-F_\star)+\sigma^2}{\sqrt{T}}.
            \]
        \end{theorem}
    
    \subsection{Complex Gradient and Wirtinger Calculus}
        There are many optimization problems performed over complex vector spaces, in which case the standard definition of a gradient becomes meaningless. 
        However, there is an appropriate replacement given by the theory of Wirtinger (or CR) calculus \cite{widrow1975complex,brandwood1983complex,vandenBos1994complex,hjorungnes2007complex,kreutzdelgado2009complex,koor2023wirtinger}, which is given by $2\nabla_\z f$ for analytic $f$ and $2\nabla_{\cl{\z}}f$ for non-analytic $f$, as introduced in \Cref{sec:notation}. This notion also extends to the bi-complex setting \cite{alpay2024extension}. When $f$ is real-valued, the direction of steepest ascent is determined by $\nabla_{\cl{\z}}f$ (resp. $\nabla_\z f)$); equivalently, the direction of steepest descent is
        $-\nabla_{\cl{\z}}f$ (resp. $-\nabla_\z f$) \cite{hjorungnes2007complex,kreutzdelgado2009complex}. 
        \begin{remark}
            It is worth emphasizing that  literature is not fully standardized.
            Different authors use different conventions for the ``complex gradient.'' Some define the descent direction by $\nabla_{\cl{\z}}f$, others by $2\nabla_{\cl{z}}f$ (resp. $\nabla_\z f$ and $2\nabla_\z f$), and others organize derivatives as row Jacobians rather than column gradients, possibly followed by transpose or Hermitian transpose operations. Consequently,the precise formulations of each algorithm may differ up to a transpose, conjugation, or a constant factor, which may be absorbed into the step size. For this reason, when reading the literature one should always check the precise
            derivative convention and update rule, rather than relying solely on the notation $\nabla_z f$ or $\nabla_{\bar z} f$ \cite{vandenBos1994complex,hjorungnes2007complex,kreutzdelgado2009complex}.
        \end{remark}
        
        The use of complex steepest-descent methods predates the recent optimization literature. Early instances appear in the complex LMS algorithm and in adaptive-array theory \cite{widrow1975complex,brandwood1983complex}, while later works developed systematic frameworks for complex gradients, Hessians, and matrix derivatives \cite{vandenBos1994complex,hjorungnes2007complex,kreutzdelgado2009complex}. Furthermore, some convergence results are available in several problem-specific complex-valued settings, including phase retrieval and complex-valued neural networks \cite{candes2015wirtinger,zhang2014convergence,xu2015convergence}.
    
    \subsection{Kernel Regression}
        \label{sec:kernel-regression}
        In many applied mathematical and physical problems, one wishes to approximate or learn a function from a finite set of observations. This can be formulated as a minimization problem over a Hilbert space $\mathcal{H}$. More precisely, we seek to recover
        \[
            f_*=\underset{f\in\mathcal{H}}{\arg\min}\,F(f),
        \]
        where $F(f)$ is a functional measuring approximation error with an optional regularization term.
        
        A particular case of interest is when $\mathcal{H}$ is a reproducing kernel Hilbert space (see \Cref{sec:rkhs} for more), in which case the problem admits a elegant solution structure. If $K$ is the reproducing kernel of $\mathcal{H}$, then the representer theorem (\Cref{thm:crt}) states that for a large class of loss functionals $F$, the solution $f_\ast$ minimizing $F$ must belong to the finite-dimensional span of the kernel sections associated with the data points, i.e. $f_\ast(\cdot) = \sum_{i=1}^n \alpha_i K(\cdot, z_i)$ for $\alpha_i\in \C$.
        
        For the particular case where $F$ is a regularized least-squares functional defined over a data set $\{(\z_i,y_i)\}_{i=1}^n$
        \begin{equation}
            \label{eq:kernel-LS-functional}
            F(f)
            =
            \sum_{i=1}^n |y_i-f(\z_i)|^2+\lambda\|f\|_{\mathcal H}^2,
            \qquad \lambda>0,
        \end{equation}
        we see that finding an $f_\ast$ minimizing $F$ is equivalent to finding an $\balpha = (\alpha_1,\dots, \alpha_n)^T$ that minimizes \Cref{eq:Falpha-kernel}, i.e.
        \begin{equation}
            \label{eq:Falpha-kernel}
            F(\balpha)
            =
            \frac{1}{2}\|\by-K\balpha\|^2+\lambda\,\balpha^*K\balpha
            =
            \frac{1}{2}(\by-K\balpha)^*(\by-K\balpha)+\lambda\,\balpha^*K\balpha,
        \end{equation}
        and then setting $f_\ast\Def \sum_{i=1}^n \alpha_i K(\cdot, \z_i)$. Here, we use $\y=(y_1,\dots,y_n)^T$ and $K=\{K(\z_i,\z_j)\}_{1\leq   i,j\leq   n}$.
        
        We note that the classical representer theorem requires a regularization term. However, in the case that we are only interested in minimizing
        \[
            F(f)=\sum_{i=1}^n |y_i-f(\z_i)|^2,
        \]
        the proof may be adapted to show that at least one $f_\ast \in \argmin_{f\in\mcal H} F(f)$ has the form $f_\ast(\cdot)=\sum_{i=1}^n \alpha_i K(\cdot, \z_i)$, in which case we may alternatively minimize
        \[
            F(\balpha)=\frac{1}{2}\|\by-K\balpha\|^2.
        \]
        Alternatively, using \Cref{prop:K_plus_lamba_I}, one may minimize a similar objective that simultaneously accounts for the regularization term:
        \[
            F(\balpha) = \frac{1}{2}\|y-(K+\lambda I)\alpha\|^2.
        \]

\section{Contributions}
    First, \Cref{sec:complex_results} presents complex-SGD and its main convergence guarantees (\Cref{thm:avg_iter_converge_complex},\Cref{thm:mu_convex_convergence_complex},\Cref{thm:stationary_adaptive}) under the mild assumptions outlined in \Cref{sec:problem_setup}. These guarantees closely parallel the corresponding results in the real setting; importantly, the new assumptions are almost identical to those typically imposed in the real-valued case.
    
    Following this, \Cref{sec:directional_bias} briefly connects complex-SGD to the random Kaczmarz method, whose convergence guarantees remain valid in the complex setting. This connection motivates a complex extension of the directional bias results for small step-sizes from \cite{dir_bias_sgd,luo2022directional}. This extension is presented in \Cref{cor:dir_bias}.
    
    Finally, \Cref{sec:applications_empirical} demonstrates the efficacy of complex-SGD in recovering the optimal functions for varying RKHS's outlined in \cite{AlpayDeMartinoDiki2025ComplexRepresenter} up to an arbitrary target tolerance. In particular, we present numerical demonstrations of the recovery of superoscillations in the Fock space (\Cref{sec:fock-superosc}), RBF supershifts in the RBF space \Cref{rmk:rbf-numerics}, and Blaschke products in the Hardy space \Cref{Hardy_sub}.
    
    All proofs and computations for the preceding results are supplied in \Cref{sec:proofs}, \Cref{sec:directional_bias_proof}, and \Cref{sec:least_squares_example}.

\section{Complex SGD}\label{sec:csgd}
    \subsection{Problem Setup}
    \label{sec:problem_setup}
        As in the classical setting we hope to minimize a loss given by a real-valued convex function $F$ of the form $F(\z)=\E_{\omega\sim \mcal D}[f(\z;\omega)]$, with appropriate conditions imposed on the components $f(\,\cdot\,;\omega)$. However, we now allow each $f(\,\cdot\,;\omega)$, and by extension $F$, to accept complex inputs. The iterates produced by the complex variant of SGD will take the same form as in the classical case:
        \begin{equation}
            \z^{(t+1)}=\z^{(t)}-\eta_t \nabla f(\z^{(t)};\omega),
            \label{eq:sgd_update_complex}
        \end{equation}
        where $\omega$ is sampled in accordance to the the distribution $\mcal D$. Of course, $\nabla f(\z;\omega)$ is now the \textit{complex} gradient of $f$. To accommodate this change, we make the following assumptions:
        \begin{adjustwidth}{2em}{2em}
            \begin{assum}[Unbiasedness]
                For all $\z\in \C^n$,
                \[
                    \E_{\omega\sim\mcal D}\bigl[\nabla f(\z;\omega)\bigr]=\nabla F(\z).
                \]
                We note that when $\mcal D$ is a finite discrete distribution, this is implied by the linearity of the complex gradient and $F(\z)=\E_{\omega\sim \mcal D}[f(\z;\omega)]$.
                \label{assum:unbiased_complex}
            \end{assum}
            
            \begin{assum}[Bounded Variance]
                There exists some $\s^2>0$ such that,
                \[
                    \E\groupb{\|\nabla f(\z;\omega)-\nabla F(\z)\|^2}\leq   \s^2.
                \]
                \label{assum:bounded_var_complex}
            \end{assum}
            
            \begin{assum}[Bounded Below]
                $F$ is bounded below, i.e. $\inf_{\z\in\C^n} F(\z)>-\infty$.
                \label{assum:bounded_below_complex}
            \end{assum}
            
            \begin{assum}[Minima is Stationary]
                For any $\zstar\in \C^n$ such that $F(\zstar)=\inf_{\z\in\C^n}F(\z)$, we have $\nabla F(\zstar)=0$.
                \label{assum:stationary_minima_complex}
            \end{assum}
            
            \begin{assum}[$L$-Lipschitz Alternative]
                There exists some constant $L>0$ such that for all $\z,\w\in \C^n$,
                \[
                    f(\w;\omega)\leq   f(\z;\omega)+\lll \nabla f(\z;\omega),\w-\z\rrr_\R+\frac{L}{2}\|\w-\z\|^2 \quad (a.s.).
                \]
                Note that this immediately implies the same inequality holds with $F$ as well.
                \label{assum:L_lip_complex}
            \end{assum}
            
            \begin{assum}[Convex Alternative]
                The components $f(\,\cdot\,;\omega)$ satisfies
                \[
                    f(\w;\omega)\geq f(\z;\omega) + \lll \nabla f(\z;\omega),\w-\z\rrr_\R \quad (a.s.).
                \]
                Note that this immediately implies the same inequality holds with $F$ as well.
                \label{assum:convex_complex}
            \end{assum}
            
            \begin{assum}[$\mu$-Strongly Convex Alternative]
                There exists some constant $\mu>0$ such that for all $\z,\w\in \C^n$,
                \[
                    F(\w)\geq F(\z) + \lll \nabla F(\z),\w-\z\rrr_\R + \frac{\mu}{2} \|\w-\z\|^2.
                \]
                \label{assum:mu_convex_complex}
            \end{assum}
        \end{adjustwidth}
        
        We note that the \Cref{assum:unbiased_complex}, \Cref{assum:bounded_var_complex}, and \Cref{assum:bounded_below_complex} are identical to the assumptions required for classical SGD. However, because the gradient is not truly gradient in the usual sense, we must impose \Cref{assum:stationary_minima_complex} to ensure that any minima is stationary. \Cref{assum:L_lip_complex}, \Cref{assum:convex_complex}, and \Cref{assum:mu_convex_complex} hold true in the real setting when $\nabla f$ is $L$-Lipschitz, $f$ is convex, and $F$ is $\mu$-strongly convex, respectively. However, these implications rely on $\nabla f$ being a gradient in the classical sense, so just as was done with the stationary minima assumption, we must impose these assumptions separately. It is natural to ask whether this collection of assumptions is satisfiable for the complex gradient. In \Cref{sec:validity_of_assumptions}, we show that the least-squares objective satisfies the full list of assumptions.
        
        As a final remark, we note that whenever $F$ and each $f(\cdot;\omega)$ only accepts real-valued inputs, the update rule \Cref{eq:sgd_update_complex} is identical to to the classical setting given by \Cref{eq:sgd_update}. Moreover, all the presented assumptions reduce to their real counterparts with $\lll \cdot,\cdot\rrr$ in place of $\lll \cdot, \cdot\rrr_\R$, whereby the classical convergence results are recovered.
    
    \subsection{Complex SGD Results}
    \label{sec:complex_results}
        In this section, we present the convergence results the assumptions of \Cref{sec:problem_setup} allow for. We note that the only difference between the following theorems and the classical counterpart presented in \Cref{sec:classical_sgd_background} lies in the hypotheses. Even so, the hypotheses are nearly identical to equivalent versions of the classical hypotheses. This is an artifact of having found a suitable definition of a gradient for the complex setting that respects similar conditions with a minor adjustment to the inner product used. In fact, the proofs of the following results are effectively compressed proofs from \cite{handbook}. The first result states that the limit of the average of the iterates produced by complex SGD minimize $F$ when we do not impose a condition analogous to strong convexity onto $F$.
        
        \begin{theorem}[Convergence in polynomial time]
        \label{thm:avg_iter_converge_complex}
            Under the assumptions of \Cref{sec:problem_setup}, except for \Cref{assum:mu_convex_complex}, if $T\geq 1$ and $\eta_t < \frac{1}{4L}$ for all $t\leq T$, then
            \[
                \E\groupb{F\groupp{\sum_{t=0}^{T-1}\frac{\eta_t}{\sum_{i=0}^{T-1}\eta_i}\z^{(t)}}-F(\zstar)} \leq   \frac{\|\z^{(0)}-\zstar\|^2}{\sum_{t=0}^{T-1} \eta_t}+\frac{2\s_\ast \sum_{t=0}^{T-1}\eta_t^2}{\sum_{t=0}^{T-1} \eta_t},
            \]
            where $\z^{(t)}$ is the $t$-th complex SGD iterate generated by \Cref{eq:sgd_update_complex} and $\s_\ast\Def \E\|\nabla f(\zstar;\omega)\|^2<\infty$.
        \end{theorem}
        
        For simplicity, \Cref{assum:bounded_var_complex} is left unchanged in the hypothesis, but the proof only requires $\E\|\nabla f(\zstar;\omega)\|^2\leq   \s^2<\infty$.
        
        In the case that an analogue to strong convexity holds, i.e. \Cref{assum:mu_convex_complex}, we see that the iterates of complex SGD converge (in expectation) exponentially to the minimizer of $F$ up to an error horizon. Again, do not need the full strength of \Cref{assum:bounded_var_complex} and can assume $\E\|\nabla f(\zstar;\omega)\|^2\leq\s^2<\infty$ instead.
        
        \begin{theorem}[Exponential Convergence]
        \label{thm:mu_convex_convergence_complex}
            Under the assumptions of \Cref{sec:problem_setup}, the sequence of iterations $\{\z^{(t)}\}_{t=0}^\infty$ generated by \Cref{eq:sgd_update_complex} with step sizes $\eta_t$ satisfying $0<c\leq   \eta_t<\frac{1}{2L}$ for all $t$ satisfies
            \[
                \E\|\z^{(t+1)}-\zstar\|^2\leq   (1-c\mu)^{t+1}\|\z^{(0)}-\zstar\|^2+\frac{\s_\ast}{2L^2c\mu},
            \]
            where $\s_\ast \Def \E\|\nabla f(\zstar;\omega)\|^2<\infty$. In the case that $\eta_t$ is constant, i.e. $\eta_t= \eta$ for all $t$, we may instead write
            \[
                \E\|\z^{(t+1)}-\zstar\|^2\leq   (1-\eta \mu)^{t+1}\|\z^{(0)}-\zstar\|^2+\frac{2\s_\ast \eta}{\mu}.
            \]
        \end{theorem}
        
        The final main result is 
        \begin{theorem}[Stationary Convergence of Complex SGD with Adaptive Step Sizes]
        \label{thm:stationary_adaptive_complex}
            Assume the assumption from Section \ref{sec:problem_setup} hold.
            Let $\{\z^{(t)}\}_{t\geq 0}$ be generated by \Cref{eq:sgd_update_complex} with stepsizes satisfying
            \[
                0<\eta_t\leq \frac{1}{L}\qquad (t\geq 0).
            \]
            Then for every $T\geq 1$,
            \[
                \sum_{t=0}^{T-1}\eta_t\,\E\|\nabla F(\z^{(t)})\|^2 \leq 2\big(F(\z^{(0)})-F_\star\big) + L\sigma^2\sum_{t=0}^{T-1}\eta_t^2,
            \]
            where $F_\star:=\inf_{\z\in\C^n}F(\z)>-\infty$.
            Consequently,
            \[
                \min_{0\leq   t\leq   T-1}\E\|\nabla F(\z^{(t)})\|^2 \leq \frac{2(F(\z^{(0)})-F_\star)+L\sigma^2\sum_{t=0}^{T-1}\eta_t^2}{\sum_{t=0}^{T-1}\eta_t}.
            \]
            In particular, if $\eta_t=\frac{1}{L\sqrt{T}}$ for $t=0,\dots,T-1$, then
            \[
                \min_{0\leq   t\leq   T-1}\E\|\nabla F(\z^{(t)})\|^2 \leq \frac{2L(F(\z^{(0)})-F_\star)+\sigma^2}{\sqrt{T}}.
            \]
        \end{theorem}

\section{Directional Bias}
\label{sec:directional_bias}
    Recent work has been done outlining the directional bias of SGD and GD for general linear regression problems in both the small learning rate regime and moderate and annealing learning rate regime \cite{dir_bias_sgd}. Other attempts at producing directional bias results, particularly in the kernel regression setting, have been made \cite{luo2022directional}. However, these results utilize the ratio
    \[
        \frac{\E\|K(\z^{(t)}-\zstar)\|^2}{\E\|\z^{(t)}-\zstar\|^2},
    \]
    to make conclusions about the direction of convergence of the iterates $\z^{(t)}$; moreover, these results require some (slightly) restrictive assumptions about the Gram matrix $K$.
    
    We offer another characterization of the directional bias of SGD in linear regression in the small learning rate regime that extends the findings of \cite{dir_bias_sgd} to the complex setting with the directional bias characterized by $\lll \z^{(t)}-\zstar,\bm{v}_\ell\rrr$, where $\bm{v}_\ell$ is the $\ell$-th right singular direction of $K$.
    
    \subsection{Main Directional Bias Result}
        Another stochastic iterative method, the Random Kaczmarz method (RK) and its variations, designed to solve large-scale linear systems has gained traction in recent years. The iterates produced take the form
        \[
            \z^{(t+1)}=\z^{(t)}+\eta \|\a_i\|^{-2}(b_i-\a_i \z^{(t)})\a_i^T,
        \]
        in the hopes to minimize the objective function
        \[
            F(\z) = \frac{1}{2}\|A\z-\b\|^2,
        \]
        where $A\in \R^{m\times n}$ for $m\geq n$ and $\a_i$ is the $i$-th row of $A$. The exponential (expected) convergence of the iterates with $\eta=1$ when the system $Ax=b$ is consistent and $A$ has full rank was first proven in \cite{RK}. Given that applying SGD to the objective function $F$ with By defining the sample gradients as
        \[
            f_i(\z) = \frac{n}{2}(\a_i \z - b_i)^2,
        \]
        so that $F(\z)=\E f_i(\z)$ when $i\sim\unif(n)$, it is clear the iterates of SGD take the form
        \[
            \z^{(t+1)}=\z^{(t)}-\eta n (b_i-\a_i \z^{(t)})\a_i^T.
        \]
        Perhaps unsurprisingly given the similarity of the update procedures for both algorithms, it has been shown that the see that RK is equivalent to SGD when using a re-weighted distribution to choose $i\sim \mcal D$ \cite{SGD_RK}.
        
        The connection between the RK and SGD in the real case, and the fact that the complex SGD iterates for when $A\in \C^{m\times n}$, i.e.
        \[
            \z^{(t+1)}=\z^{(t)}-\eta n (b_i-\a_i \z^{(t)})\a_i\dl,
        \]
        differs from the classical SGD iterates by a conjugation in the update vector suggests that we may borrow the analysis utilized for the directional bias of RK \cite{RK_small_sing}! Indeed, this is how the following (complex) SGD directional bias result is proved:
        
        \begin{cor}[Directional Bias of SGD in Kernel Regression with Small Learning Rate]
            \label{cor:dir_bias}
            Given system $A\z=\b$ with $A\in \C^{m\times n}$, the iterates of (complex) SGD applied to the least squares objective $F(\z)=\frac{1}{2}\|A\z-\b\|^2$ satisfy the following inequality,
            \[
                \groupabs{\E \lll \z^{(t+1)}-\zstar,\bm{v}_k\rrr -\groupp{1-\frac{\eta n \s_k^2}{m}}^{t+1}\lll \z^{(0)}-\zstar,\bm{v}_k\rrr }\leq   \|\bm{\ve}\|,
            \]
            where $\bm{\ve}=A\zstar-\b$ for $\zstar\in\argmin \frac{1}{2}\|A\z-\b\|^2$. In particular, if the system $A\x=\b$ is consistent, we have
            \[
                \E \lll \z^{(t+1)}-\zstar,\bm{v}_k\rrr =\groupp{1-\frac{\eta n \s_k^2}{m}}^{t+1}\lll \z^{(0)}-\zstar,\bm{v}_k\rrr.
            \]
        \end{cor}
        For consistent systems in particular, we observe that the iterates converge to the optimum in the direction of the smallest singular value. This is indeed the same conclusion as in \cite{RK_small_sing} for RK and in \cite{dir_bias_sgd} for (real) SGD with small step size.

\section{Applications and Empirical Results}
\label{sec:applications_empirical}
    Before presenting the numerical experiments, we emphasize that the complex setting is not just a cosmetic extension of the real-valued algorithm. In the examples below, both the coefficient vector $\balpha_*$ and the data vector $\bw$ are genuinely complex (see \Cref{sec:NE_blaschke}). Thus, the stochastic iteration must be carried out directly in the complex space $\mathbb C^{n}$. This is precisely where the Wirtinger derivative is important: it yields the correct gradient update for a real-valued loss function defined for complex variables.
    
    A purely real formulation would require splitting the problem into real and imaginary parts. Although this is possible in principle, it hides the natural structure of the Hardy-space recovery problem. By contrast, the complex-SGD method treats the real and imaginary parts together, in a coupled way, and respects the geometry of the complex RKHS. For this reason, the experiment below is a genuine test of the complex-SSGD method, rather than a real algorithm rewritten in complex notation.
    
    \subsection{Fock space and Superoscillations}
    \label{sec:fock-superosc}
        A classical principle in Fourier analysis asserts that the highest frequency present in a wave is determined by the largest frequency appearing in its Fourier decomposition. Superoscillations are a surprising exception to this usual idea: a function may exhibit local oscillations at a frequency exceeding that of each of its individual Fourier components. This phenomenon was first studied in \cite{aharonov2017mathematics} in connection with weak values in quantum mechanics. Since then, applications in signal processing, wave propagation, and quantum theory have been found. This was also later extended to generalized Fock spaces based on fractional derivatives \cite{alpay2026fractional}.
        
        The superoscillation sequence of functions is defined as
        \begin{equation}
            \label{eq:Fn-superosc}
            F_n(a,z):=\left(\cos\frac{z}{n}+ia\sin\frac{z}{n}\right)^n, \qquad n\in\mathbb N,\quad a>1,\quad z\in\mathbb C.
        \end{equation}
        Restricting to the real axis recovers the usual real-variable superoscillatory sequence.
        
        For each fixed $z\in\mathbb C$, the supershift property is given by 
        \[
            \lim_{n\to\infty}F_n(a,z)=e^{iaz}.
        \]
        When $a=1$, the expression reduces identically to $e^{iz}$ for every $n$. When $a>1$, the limit has frequency $a$, whereas each Fourier coefficient appearing in $F_n(a,\cdot)$ has frequency bounded in magnitude by $1$. This is the superoscillatory, or supershift, behavior (see Proposition~\ref{prop:superosc} and Remark~\ref{rmk:so}).
        
        The superoscillatory sequence may be naturally interpreted in the Bargmann-Fock space, as we will see soon. Therefore, we recall the following definition.
        
        \begin{definition}[Fock space]\label{def:fock}
            The Bargmann-Fock space $\mathcal F$ is the reproducing kernel Hilbert space of entire functions $f$ such that
            \begin{equation}
                \label{eq:fock-space-def}
                \mathcal F = \left\{f\in\mathcal H(\mathbb C):\frac{1}{\pi}\iint_{\mathbb C}|f(z)|^2 e^{-|z|^2}\,dA(z)<\infty
                \right\},
            \end{equation}
            with the associated reproducing kernel
            \begin{equation}
                \label{eq:fock-kernel}
                B(z,w)=e^{z\overline w}, \qquad z,w\in\mathbb C.
            \end{equation}
            Here $dA(z)=dx\,dy$ denotes planar Lebesgue measure for $z=x+iy$.
        \end{definition}
        
        Moreover, up to a positive multiplicative constant, it is the unique Hilbert space of entire functions for which
        \begin{equation}
            \label{eq:fock-adjoint}
            \partial_z^{*}=M_z,
        \end{equation}
        where $\partial_z=\frac{d}{dz}$ and $M_z$ denotes multiplication by $z$.
        
        The connection between superoscillations and the Fock space was established in \cite{alpay2023superoscillations}:
        
        \begin{proposition}[{\cite[Proposition 4.1]{alpay2023superoscillations}}]
            \label{prop:superosc}
            For every $n\in\mathbb N$ and $a>1$,
            \begin{equation}
                \label{eq:superosc-kernel-rep}
                F_n(a,z)
                =
                \sum_{j=0}^n C_j(n,a)e^{iz(1-2j/n)}
                =
                \sum_{j=0}^n C_j(n,a)\,B(z,z_j),
            \end{equation}
            where
            \begin{equation}
                \label{eq:cjzj}
                C_j(n,a)
                =
                \binom{n}{j}
                \left(\frac{1+a}{2}\right)^{n-j}
                \left(\frac{1-a}{2}\right)^j,
                \qquad
                z_j=-i\left(1-\frac{2j}{n}\right),
                \quad j=0,\dots,n.
            \end{equation}
        \end{proposition}
        
        Moreover, the representation of superoscillations in terms of the Fock space kernel, and its possible usefulness, was first pointed out in \cite[Remark 4.2]{alpay2023superoscillations}. 
        
        \begin{remark}
            \label{rmk:so}
            The Fourier component has the form $e^{i\omega z}$, where $\omega$ is called the frequency.
            The superoscillatory behavior is apparent in the representation \eqref{eq:superosc-kernel-rep}: each Fourier component of $F_n(a,\cdot)$ has frequency $1-\frac{2j}{n}$, which is bounded by $1$ in magnitude, whereas the limit $e^{iaz}$ has frequency $a$, which may be arbitrarily large.
        \end{remark}
        
        It was then studied more explicitly in \cite{AlpayDeMartinoDiki2025ComplexRepresenter}, where the authors identified an optimization problem for which the superoscillatory sequence arises as a solution:
        
        \begin{theorem}
        [{\cite[Theorem 3.9]{AlpayDeMartinoDiki2025ComplexRepresenter}}]
            \label{thm:superoscillation-minimizer}
            For a fixed $a>1$ Let $\{(z_k,w_k)\}_{k=0}^{n}$ be set of complex data points given by
            \[
                z_k=-i\left(1-\frac{2k}{n}\right),
            \]
            and
            \begin{equation}
                \label{eq:sup_wk}
                w_k(n,a)=\left(\frac{1+a}{2}\right)^n\left[ e^{1-\frac{2k}{n}}\left(1+\frac{1-a}{1+a}e^{-\frac{2}{n}+\frac{4k}{n^2}}\right)^n+\lambda \binom{n}{k}\left(\frac{1-a}{1+a}\right)^k\right].
            \end{equation}
            For $\lambda>0$, the superoscillation sequence $F_n(a,\cdot)$ minimizes the regularized least-squares objective $F$:
            \[
                F_n(a,\cdot)=\arg\min_{f\in\mathcal F} F(f),
            \]
            where $F$ is the kernel least-squares regression function defined by \Cref{eq:kernel-LS-functional}.
        \end{theorem}
        
        \subsubsection{Fock Space Numerical Results}
        \label{sec:NE_suposcillation}
            The purpose of this numerical experiment is twofold: first, we confirm that the complex-SGD method based on the Wirtinger derivative correctly approximates the coefficient vector predicted by the \Cref{thm:superoscillation-minimizer}, and second, we illustrate the superoscillatory behavior of the resulting recovered function on the real axis.
            
            For the numerical experiment, we will use the data points $\{z_k,w_k\}_{i=0}^n$ defined in \Cref{thm:superoscillation-minimizer} using parameters
            \[
                n=40,\qquad a=2,\qquad \lambda=1.
            \]
            By equation \eqref{eq:superosc-kernel-rep}, the superoscillation sequence
            has the kernel representation
            \[
                F_n(a,z)=\left(\cos\frac{z}{n}+ia\sin\frac{z}{n}\right)^n =\sum_{j=0}^{n} C_j(n,a)\,B(z,z_j),
            \]
            where $C_j$ and $z_j$ are given by \eqref{eq:cjzj}. Hence Theorem~\ref{thm:superoscillation-minimizer} and the representer theorem state that the coefficient vector of the kernel representation of the minimizer is 
            \begin{equation}
                \label{eq:balpha}
                \balpha_*=(C_0(n,a),\dots,C_n(n,a))^T.
            \end{equation}
            To generate \Cref{fig:cv_rate}, we first build the data vector
            \[
                \bw=(w_0,\dots,w_n)^T
            \]
            from \Cref{thm:superoscillation-minimizer}, and then we generate a sample of complex-SGD iterates for the least-squares objective 
            \[
                \text{minimize}\quad \frac{1}{2}\|(K+\lambda I)\balpha - \w\|^2,
            \]
            suggested by \Cref{prop:K_plus_lamba_I} using a normalized update rule similar to \Cref{eq:sgd_update_complex} with gradient given in \Cref{lem:comp_sgd_update_LS}. More specifically, the updates are given by 
            \begin{equation}
                \label{eq:alpha_step}
                \balpha^{(t+1)}=\balpha^{(t)}-\eta_t\,\frac{\a_{i_t}\balpha^{(t)}-w_{i_t}}{a_{i_t}\|_2^2}\,\a_{k_t}\dl,
            \end{equation}
            where $\a_{i_t}$ is the $i_t$-th row of $K+\lambda I$ and $i_t\sim \unif([n+1])$. This change in the update rule from \Cref{eq:sgd_update_complex} is purely to improve numerical stability, and the resulting figures (\Cref{fig:cv_rate} and \Cref{fig:exact_vs_approx}) are virtually identical to those with the un-normalized update rule. Finally, we plot the relative residual and approximation error at each iteration. 
            
            \Cref{fig:exact_vs_approx} displays coefficients of $F_{40}(2,x)$, i.e. $\balpha_\ast$, alongside the coefficients from the final iterate $\balpha^{(100000)}$ from \Cref{fig:cv_rate}. 
            
            \begin{figure}[ht]
                \centering
                \begin{subfigure}[t]{0.48\linewidth}
                    \centering
                        \begin{tikzpicture}
                            \begin{axis}[
                            width=0.95\textwidth,
                            height=0.75\textwidth,
                            xlabel={Iteration},
                            ylabel={Error},
                            ymode=log,
                            legend pos=north east,
                            title={Fock: Complex-SGD Error}
                            ]
                                \addplot[thick, blue] table[col sep=comma, x=iteration, y=relative_residual] {fock/plot1_convergence.csv};
                                \addlegendentry{\footnotesize{Relative Residual Error}}
                                
                                \addplot[thick, red] table[col sep=comma, x=iteration, y=relative_coefficient_error] {fock/plot1_convergence.csv};
                                \addlegendentry{\footnotesize{Relative Approximation Error}}
                            \end{axis}
                        \end{tikzpicture}
                    \caption{Relative residual $\|(K+\lambda I)\balpha^{(t)}-\bw\|_2/\|\bw\|_2$ and relative coefficient error $\|\balpha^{(t)}-\balpha_*\|_2/\|\balpha_*\|_2$ versus iteration.}
                    \label{fig:cv_rate}
                \end{subfigure}
                \hfill
                \begin{subfigure}[t]{0.48\linewidth}
                    \centering
                        \begin{tikzpicture}
                            \begin{axis}[
                            width=0.95\textwidth,
                            height=0.75\textwidth,
                            xlabel={Index $j$},
                            ylabel={Coefficient value},
                            ytick={-1.5e11,0,1.5e11},
                            scaled y ticks = false,
                            legend pos=north east,
                            title={Fock: Exact vs. Recovered Coefficients}
                            ]
                            
                                \addplot[only marks, mark=square*, blue, mark options={fill=blue}] 
                                table[col sep=comma, x=j, y=exact_real] {fock/plot2_coefficients.csv};
                                \addlegendentry{\footnotesize{Exact Coefficient}}
                                
                                \addplot[thick, only marks, red, mark=x] 
                                table[col sep=comma, x=j, y=recovered_real] {fock/plot2_coefficients.csv};
                                \addlegendentry{\footnotesize{Recovered Coefficient}}
                            \end{axis}
                        \end{tikzpicture}
                    \caption{Comparison of the entries of $\balpha^{(10000)}$ and $\balpha_\ast$.}
                    \label{fig:exact_vs_approx}
                \end{subfigure}
                \caption{Numerical verification of the complex SGD recovery in the Bargmann--Fock setting. Panel (a) shows convergence at the coefficient level, while panel (b) compares the entries of $\balpha^{(100000)}$ and $\balpha_\ast$.}
                \label{fig:sgd_two_panels}
            \end{figure}
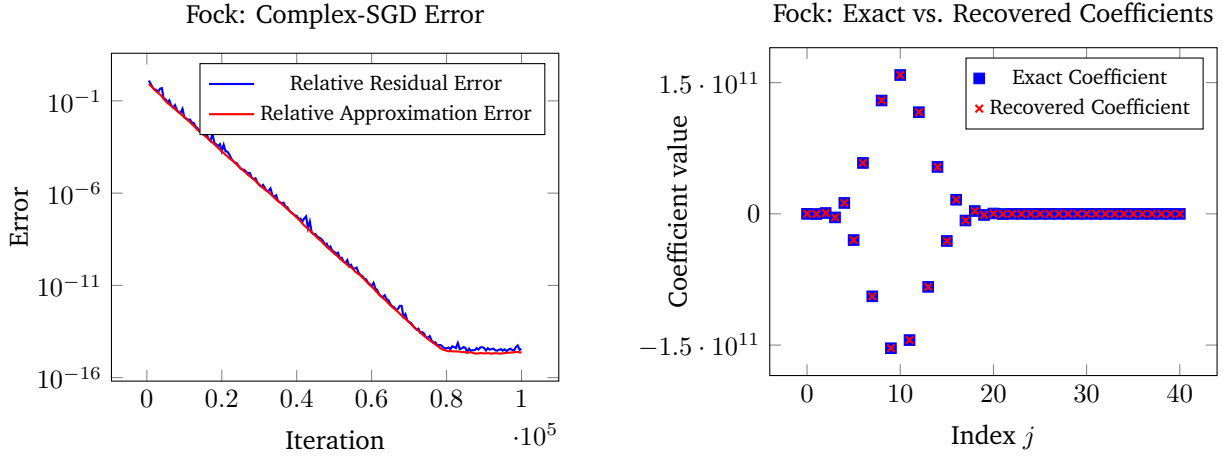
            
            Since the recovered coefficients in \Cref{fig:sgd_two_panels} match nearly perfectly, in \Cref{fig:reconstructed_superoscillation_vs_limit} we compare the limiting function $e^{2ix}$, the approximated reconstructed function $F^{{\rm{approx}}}= K\balpha^{(100000)}$, and $F_{40}(2,x)$, which has a closed form given by \Cref{eq:Fn-superosc}. This simultaneously illustrates the superoscillatory phenomenon and the quality of the reconstruction $F^{{\rm{approx}}}$ compared to $F_{40}(2,x)$. Indeed, the superoscillatory phenomenon (\Cref{rmk:so}) is present because $F_{40}(2,x)$ (and consequently $F^{\rm approx}$) is constructed as a linear combination of Fourier components whose frequencies are all bounded by $1$, whereas the limiting function $e^{2ix}$ has frequency $2$. Succinctly, the approximating functions are formed from bounded-frequency components, but their limit exhibits a higher frequency. \Cref{fig:reconstructed_superoscillation_vs_limit} also shows that $F^{{\rm{approx}}}$ and $F_{40}(2,x)$ agree almost perfectly on the presented domain, demonstrating complex-SGD's ability to provide a reliable reconstruction at the function level as well. 
            
            We observe the approximation of $e^{2ix}$ given by $F^{{\rm{approx}}}$ is strongest near the origin and worsens as $|x|$ increases. Given the quality of the approximation $F^{{\rm{approx}}}$ to $F_{40}(2,x)$, this is a result of the general fact that $F_n(a,x)$ converges locally uniformly, but not uniformly, to $e^{iax}$. 
            
            \begin{figure}[ht]
                \centering
                \begin{subfigure}[t]{0.48\linewidth}
                    \centering
                        \begin{tikzpicture}
                            \begin{axis}[
                            width=0.95\textwidth,
                            height=0.75\textwidth,
                            xlabel={$x$},
                            ylabel={Real Part},
                            legend pos=north east,
                            title={Fock: Real Part Comparison}
                            ]
                                
                                \addplot[thick, blue] 
                                table[col sep=comma, x=x, y=closed_real] 
                                {fock/plot3_real_part.csv};
                                \addlegendentry{\footnotesize{Closed Form}}
                                
                                \addplot[thick, red, dashed] table[col sep=comma, x=x, y=sgd_real] 
                                {fock/plot3_real_part.csv};
                                \addlegendentry{\footnotesize{SGD}}
                                
                                \addplot[thick, dotted] table[col sep=comma, x=x, y=limit_real] 
                                {fock/plot3_real_part.csv};
                                \addlegendentry{\footnotesize{Limit}}
                            \end{axis}
                        \end{tikzpicture}
                    \caption{Real part.}
                    \label{fig:reconstructed_real}
                \end{subfigure}
                \hfill
                \begin{subfigure}[t]{0.48\linewidth}
                    \centering
                        \begin{tikzpicture}
                            \begin{axis}[
                            width=0.95\textwidth,
                            height=0.75\textwidth,
                            xlabel={$x$},
                            ylabel={Imaginary Part},
                            legend pos=north east,
                            title={Fock: Imaginary Part Comparison}
                            ]
                                
                                \addplot[thick, blue] 
                                table[col sep=comma, x=x, y=closed_imag] 
                                {fock/plot4_imag_part.csv};
                                \addlegendentry{\footnotesize{Closed Form}}
                                
                                \addplot[thick, red, dashed] table[col sep=comma, x=x, y=sgd_imag] 
                                {fock/plot4_imag_part.csv};
                                \addlegendentry{\footnotesize{SGD}}
                                
                                \addplot[thick, dotted] table[col sep=comma, x=x, y=limit_imag] 
                                {fock/plot4_imag_part.csv};
                                \addlegendentry{\footnotesize{Limit}}
                            \end{axis}
                        \end{tikzpicture}
                    \caption{Imaginary part.}
                    \label{fig:reconstructed_imag}
                \end{subfigure}
                \caption{Comparison of the real and imaginary parts of $F_{n}(a,x)$, $F^{{\rm{approx}}}$, and $e^{iax}$ with parameters $n=40$ and $a=2$. We note that in both the real and imaginary parts, the reconstructed function $F^{{\rm{approx}}}$ is visually indistinguishable from the closed-form expression $F_{40}(2,x)$.}
                \label{fig:reconstructed_superoscillation_vs_limit}
            \end{figure}
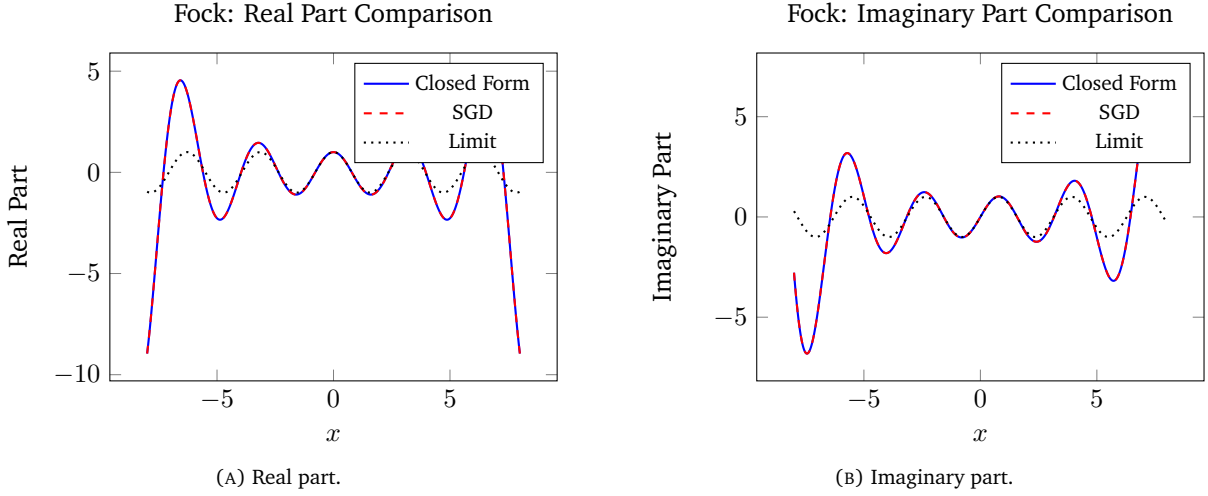
            
            These figures reasonably reinforce the results of Theorem~\ref{thm:superoscillation-minimizer}, and hence provide empirical validation of complex-SGD method in the Fock space setting. Indeed, the relative residual and approximation error both decrease to nearly machine precision, and the function reconstructed from the generated coefficients agrees closely with the exact superoscillatory sequence. 
            
            \begin{remark}[Gaussian RBF analogue of the Fock space]\label{rmk:rbf-numerics}
                We include the Gaussian RBF version of the supershift property for readers with a background in machine learning. Indeed, the Gaussian RBF kernel is standard in machine learning and is often used in methods such as support vector machines and Gaussian process models; see, for example, \cite{ScholkopfSmola2002,HofmannScholkopfSmola2008,RasmussenWilliams2006}. Its length-scale parameter determines how strongly nearby sample points interact. Even though the resulting finite-dimensional recovery problem is algebraically the same as in the Fock-space case.
                
                For $\gamma>0$, the Gaussian RBF kernel is
                \[
                    K_\gamma(z,w)=\exp\!\left(-\frac{(z-\overline w)^2}{\gamma^2}\right),\qquad z,w\in\mathbb C,
                \]
                and its reproducing kernel Hilbert space is
                \[
                    \mathcal H_\gamma^{RBF}(\mathbb C)=\left\{ f\in\mathcal H(\mathbb C): \left(\frac{2}{\pi\gamma^2}\right) \iint_{\mathbb C} |f(z)|^2 \exp\!\left(\frac{(z-\overline z)^2}{\gamma^2}\right)\,dA(z)<\infty\right\}.
                \]
                In this paper we take $\gamma=\sqrt2$. Then
                \[
                    K_{\sqrt2}(z,w)=e^{-(z^2+\overline w^{\,2})/2}B(z,w),
                \]
                so this Gaussian RBF setting can be viewed as a weighted version of the Fock-space setting.
                
                The first-type RBF supershift is
                \[
                    R_n(z,a)=\sum_{j=0}^{n} C_j(n,a)\,e^{\overline z_j^{\,2}/2}\,K_{\sqrt2}(z,z_j),
                \]
                where $z_j$ and $C_j(n,a)$ are given by \eqref{eq:cjzj}. Using the property 
                \[
                    F_n(z,a) = e^{z^2/2}R_n(z,a),
                \]
                it has been shown that the minimizer $R_n(z,a)$ of the regularized-least-squares objective in \Cref{thm:superoscillation-minimizer} with the weights $\w$ given by
                \[
                    w_k
                    =
                    e^{\frac12\left(1-\frac{2k}{n}\right)^2}
                    \sum_{j=0}^{n}
                    C_j(n,a)\,
                    e^{-\left(1-\frac{2k}{n}\right)\left(1-\frac{2j}{n}\right)}
                    +
                    \lambda C_k(n,a)\,
                    e^{-\frac12\left(1-\frac{2k}{n}\right)^2},
                \]
                has coefficients
                \[
                    \bbeta_*=(\beta_0,\dots,\beta_n)^T,
                    \qquad
                    \beta_j=C_j(n,a)e^{\overline z_j^{\,2}/2},
                \]
                in its kernel representation (see \cite[Theorem 3.20]{AlpayDeMartinoDiki2025ComplexRepresenter}).
                
                Then, just as in the Fock space example, the coefficient recovery problem reduces to the finite-dimensional complex linear system
                \[
                    \bw=(K+\lambda I)\bbeta,\qquad K_{k,j}=K_{\sqrt2}(z_k,z_j).
                \]
                with solution $\beta_\ast$ by \Cref{prop:K_plus_lamba_I}.
                
                Because of how closely related this Gaussian RBF example is to the Fock space example, the corresponding numerical results look remarkably similar to \Cref{fig:reconstructed_superoscillation_vs_limit}. As such, we omit these plots in particular, but we note that the same complex-SGD recovery method works for the Gaussian RBF kernel as well.
            \end{remark}
    
    \subsection{Hardy Space and Blaschke Products}
    \label{Hardy_sub}
    
        In this section, we present a similar application with the Hardy space over the unit disk in $\C$ being the RKHS of choice. Here, the target minimizer is a finite Blaschke product. We begin by recalling the definition of the Hardy space over the unit disk $\mathbb D$.
        
        \begin{definition}[Hardy space]
            \label{def:hardy}
            The Hardy space $\mathbf H^2(\mathbb D)$ consists of holomorphic functions $f$ on $\mathbb D$ such that
            \[
                \mathbf H^2(\mathbb D)=\left\{ f\in\mathcal{H}(\mathbb{D}):\lim_{r\to 1}\frac{1}{2\pi}\int_0^{2\pi}|f(re^{it})|^2\,dt<\infty\right\}.
            \]
            It is an RKHS with the reproducing kernel
            \[
                K(z,w)=\frac{1}{1-z\overline w} = \sum_{m=0}^{\infty} z^m \overline w^{\,m},\qquad z,w\in\mathbb D.
            \]
        \end{definition}
        
        We refer to \cite{duren1970theory} for more information on the Hardy space. An important tool in the study of $\mathbf H^2(\mathbb D)$ is provided by the Blaschke factors and the associated Blaschke products. These play a fundamental role in invariant subspace theory and interpolation; see \cite{DD,GMW,R} for more. Moreover, a new characterization of the Hardy space can be found in \cite{alpay2023new}.
        
        \begin{definition}
            Let $\{a_j\}_{j=1}^n\subset \mathbb D$. For $z\in\mathbb C$ and each $a_j\in\mathbb D$, the Blaschke factor corresponding to $a_j$ is
            \[
                \frac{z-a_j}{1-\overline{a_j}z}.
            \]
            The corresponding finite Blaschke product is
            \[
                B_n(z)=\prod_{j=1}^n \frac{z-a_j}{1-\overline{a_j}z}.
            \]
        \end{definition}
        
        \begin{remark}
            The collection of all Blaschke factors are the automorphisms of $\mathbb D$ up to a unimodular constant. Moreover, on the unit circle one ($z\in\partial \mathbb{D})$ has $|B_n(z)|=1$, so $B_n$ is one of the simplest $\mathbf H^2$-inner functions. Its zeros are precisely the points $\{a_j\}_{j=1}^n$.
        \end{remark}
        
        The derivative of $B_n$ is given by
        \[
            B_n'(z)=\sum_{l=1}^n\frac{1-|a_l|^2}{(1-\overline{a_l}z)^2}\left( \prod_{\substack{k=1\\k\neq l}}^n\frac{z-a_k}{1-\overline{a_k}z}\right),
        \]
        as seen in \cite{GMW}. In particular,
        \[
            B_n'(a_j)=\frac{1}{1-|a_j|^2}\prod_{\substack{k=1\\k\neq j}}^n\frac{a_j-a_k}{1-\overline{a_k}a_j}.
        \]
        Assuming the zeros of $B_n$ are simple and distinct from the origin, the Blaschke product admits the kernel expansion
        \[
            B_n(z)=c_0+\sum_{j=1}^n c_j K(z,a_j),\qquad c_0=\frac{1}{B_n(0)}, \qquad c_j=\frac{1}{a_j B_n'(a_j)},
        \]
        given in \cite{FL}, where the identity was proved in a more general setting.
        Sometimes this is referred to as the Bedrosian identity (see \cite{CQK}).
        
        To formulate a numerical experiment, we wish to use an optimization problem for which the Blaschke product is the minimizer. We recall the following:
        \begin{theorem}
        [{\cite[Proposition 6.7]{AlpayDeMartinoDiki2025ComplexRepresenter}}]
            \label{Hardy}
            Suppose the data set $\{(a_k,w_k)\}_{k=1}^n$ consists of the simple nonzero zeros $a_1,\dots,a_n\in\mathbb D$ of $B_n$ and the associated values
            \[
                w_k=\lambda \frac{1}{\overline{a_k B_n'(a_k)}}-\frac{1}{\overline{B_n(0)}},\qquad 1\leq   k\leq   n.
            \]
            Then the Blaschke product $B_n$ is the regularized least-squares minimizer
            \[
                B_n = f_\ast=\arg\min_{f\in \mathbf H^2(\mathbb D)} F(f),
            \]
            where $F$ is given by \eqref{eq:kernel-LS-functional}.
        \end{theorem}
        
        \subsubsection{Hardy space Numerical Experiment}
        \label{sec:NE_blaschke}
            Similarly to \Cref{sec:NE_suposcillation}, we numerically verify Theorem~\ref{Hardy} in a complex Hardy-space example to show complex-SGD based on the Wirtinger gradient performs as desired empirically. Importantly, because the roots, and hence the coefficients, are complex, \textit{a real-valued stochastic gradient method is not sufficient} without modification to the original problem, justifying the use of complex-SGD.
            
            For the following numerical experiments, we first generate fifty random complex roots $\{a_1,\dots, a_{50}\}$ inside the unit disk and fix $\lambda=1$. With this, $n={50}$ and the corresponding Blaschke product is
            \[
                B_{50}(z)=\prod_{j=1}^{50} \frac{z-a_j}{1-\overline{a_j}z}.
            \]
            
            The kernel representation of $B_n$ may be expressed as
            \[
                B_{50}(z)=c_0+\sum_{j=1}^{50} c_j K(z,a_j),
            \]
            where
            \[
                c_0=\frac{1}{\cl{B_{50}(0)}},\qquad c_j=\frac{1}{\cl{a_j B_{50}'(a_j)}}, \qquad j=1,\dots,{50}.
            \]
            Hence, the exact coefficient vector we wish to recover is
            \begin{equation}
                \label{eq:alpha_hardy}
                \balpha_*=(c_1,\dots,c_{50})^T\in \mathbb C^{50}.
            \end{equation}
            The representation contains the additional constant term $c_0$. In the numerical implementation, this constant is computed explicitly from $B_{50}(0)$, while the stochastic method is used to recover the coefficients $c_1,\dots,c_{50}$. The corresponding data points from \Cref{Hardy} are given by
            \[
            w_k = \lambda c_k-c_0.
            \]
            Again, we use \Cref{prop:K_plus_lamba_I} to reduce the objective to minimizing
            \[
                \frac{1}{2}\|(K+\lambda I)\balpha - \w\|^2,
            \]
            to recover $\balpha_\ast$ and produce the complex-SGD iterates in exactly the same way as in \Cref{sec:NE_suposcillation}. 
            
            Figure \ref{fig:hardy_complex_cv_rate} shows the relative residual and approximation error as functions of the iteration number. This gives a direct numerical verification that the complex-SGD iterates converge toward the exact coefficients from Theorem \ref{Hardy}.
            Figure \ref{fig:hardy_complex_coefficients_plane} overlays the exact coefficients and the recovered coefficients as an alternative demonstration of the reconstruction quality. In \Cref{fig:hardy_complex_cv_rate}, both the relative residual and the relative coefficient error decrease steadily to machine precision, demonstrating the linear convergence of the complex SGD. In \Cref{fig:hardy_complex_coefficients_plane}, the reconstructed coefficients are visually indistinguishable from the exact coefficients on the complex plane, confirming that the method accurately recovers the coefficient vector $\balpha_\ast$  in the Blaschke-product expansion give by \eqref{eq:alpha_hardy}.
            
            \begin{figure}[ht]
                \centering
                \begin{subfigure}[t]{0.48\textwidth}
                    \centering
                        \begin{tikzpicture}
                            \begin{axis}[
                            width=0.95\textwidth,
                            height=0.75\textwidth,
                            xlabel={Iteration},
                            ylabel={Error},
                            ymode=log,
                            legend pos=north east,
                            title={Hardy Example: Complex-SGD Error}
                            ]
                            
                                \addplot[thick, blue] table[col sep=comma, x=iteration, y=relative_residual] {hardy/Hardy_plot2_convergence.csv};
                                \addlegendentry{\footnotesize{Relative Residual Error}}
                                
                                \addplot[thick, red] table[col sep=comma, x=iteration, y=relative_coefficient_error] {hardy/Hardy_plot2_convergence.csv};
                                \addlegendentry{\footnotesize{Relative Approximation Error}}
                            \end{axis}
                        \end{tikzpicture}
                    \caption{Relative residual $\|(K+\lambda I)\balpha^{(t)}-\bw\|_2/\|\bw\|_2$ and relative approximation error $\|\balpha^{(t)}-\balpha_*\|_2/\|\balpha_*\|_2$ at each iteration.}
                    \label{fig:hardy_complex_cv_rate}
                \end{subfigure}
                \hfill
                \begin{subfigure}[t]{0.48\textwidth}
                    \centering
                    \begin{tikzpicture}
                            \begin{axis}[
                            width=0.95\textwidth,
                            height=0.75\textwidth,
                            xlabel={Real Part},
                            ylabel={Imaginary Part},
                            legend pos=south west,
                            title={Hardy: Exact vs. Recovered Coefficients}
                            ]
                            
                                \addplot[only marks, mark=square*, blue, mark options={fill=blue}] 
                                table[col sep=comma, x=exact_real, y=exact_imag] {hardy/Hardy_plot3_coefficients_complex_plane.csv};
                                \addlegendentry{\footnotesize{Exact Coefficient}}
                                
                                \addplot[thick, only marks, red, mark=x] 
                                table[col sep=comma, x=recovered_real, y=recovered_imag] {hardy/Hardy_plot3_coefficients_complex_plane.csv};
                                \addlegendentry{\footnotesize{Recovered Coefficient}}
                            \end{axis}
                        \end{tikzpicture}
                    \caption{Comparison of the exact and reconstructed coefficients in the Blaschke-product expansion.}
                    \label{fig:hardy_complex_coefficients_plane}
                \end{subfigure}
                \caption{Numerical verification of the complex-SGD recovery in the complex-root Hardy space setting.}
                \label{fig:hardy_complex_numerics}
            \end{figure}
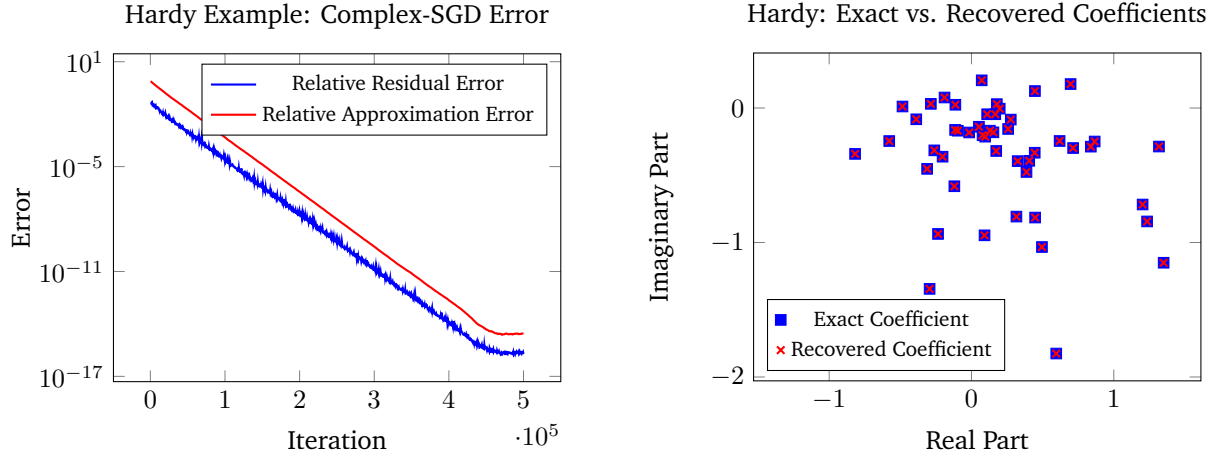
            
            To demonstrate the quality of the reconstruction $\balpha^{(500000)}$ at the function level in \Cref{fig:hardy_unit_circle_comparison}, we compare the exact Blaschke product with the function defined by the coefficients $\balpha^{(500000)}$ on the unit circle $z=e^{it}, t\in[0,2\pi]$. We will denote the reconstructed function by 
            \[
                B_{50}^{\mathrm{SGD}}(z)=c_0+\sum_{j=1}^{50} \balpha_j^{(500000)} K(z,a_j).
            \]
            \Cref{fig:hardy_unit_circle_real} and \Cref{fig:hardy_unit_circle_imag} show that the real and imaginary parts, respectively, of the reconstructed function $B_{50}^{\mathrm{SGD}}$ are visually indistinguishable from those of the exact Blaschke product. \Cref{fig:hardy_unit_circle_modulus} compares the the modulus of each function on the boundary of $\mathbb D$. Because $B_{50}$ is an inner function, we know $|B_{50}(e^{it})|=1$. We observe that the reconstruction almost perfectly mimics this behavior; that is, for the points $e^{i\theta_k}$ sampled from the unit circle, we have
            \[
                |B_{50}^{\mathrm{SGD}}(e^{i\theta_k})-1|\leq   \ve_k,
            \]
            for some choice of $0<\ve_k<5\cdot 10^{-14}$.
            
            \begin{figure}[h!]
                \centering
                \begin{subfigure}[t]{0.48\textwidth}
                    \centering
                    \begin{tikzpicture}
                        \begin{axis}[
                        width=0.95\textwidth,
                        height=0.6\textwidth,
                        xlabel={$\theta$},
                        legend pos=north east,
                        title={Real Part on $\del \mathbb D$}
                        ]
                        
                            \addplot[thick, blue] 
                            table[col sep=comma, x=theta, y=exact_real] 
                            {hardy/Hardy_plot4_unit_circle_real_part.csv};
                            \addlegendentry{\footnotesize{$\Re B_n(e^{i\theta})$}}
                
                            \addplot[thick, red, dashed] 
                            table[col sep=comma, x=theta, y=recovered_real] 
                            {hardy/Hardy_plot4_unit_circle_real_part.csv};
                            \addlegendentry{\footnotesize{$\Re B_n^{\mathrm{SGD}}(e^{i\theta})$}}
                        \end{axis}
                    \end{tikzpicture}
                    \caption{Real part on the unit circle.}
                    \label{fig:hardy_unit_circle_real}
                \end{subfigure}
                \hfill
                \begin{subfigure}[t]{0.48\textwidth}
                    \centering
                    \begin{tikzpicture}
                        \begin{axis}[
                        width=0.95\textwidth,
                        height=0.6\textwidth,
                        xlabel={$\theta$},
                        legend pos=north east,
                        title={Imaginary Part on $\del \mathbb D$}
                        ]
                        
                            \addplot[thick, blue] 
                            table[col sep=comma, x=theta, y=exact_imag] 
                            {hardy/Hardy_plot5_unit_circle_imaginary_part.csv};
                            \addlegendentry{\footnotesize{$\Im B_n(e^{i\theta})$}}
                
                            \addplot[thick, red, dashed] 
                            table[col sep=comma, x=theta, y=recovered_imag] 
                            {hardy/Hardy_plot5_unit_circle_imaginary_part.csv};
                            \addlegendentry{\footnotesize{$\Im B_n^{\mathrm{SGD}}(e^{i\theta})$}}
                        \end{axis}
                    \end{tikzpicture}
                    \caption{Imaginary part on the unit circle.}
                    \label{fig:hardy_unit_circle_imag}
                \end{subfigure}
                \hfill
                \begin{subfigure}[t]{0.95\textwidth}
                    \centering
                        \begin{tikzpicture}
                            \begin{axis}[
                                width=0.95\textwidth,
                                height=0.25\textwidth,
                                xlabel={$\theta$},
                                ylabel={Magnitude},
                                ymode=log,
                                legend pos=north east,
                                title={Boundary modulus on $\partial\mathbb D$},
                                tick label style={font=\small},
                                label style={font=\small},
                                title style={font=\small},
                            ]
                            
                                \addplot[thick, blue]
                                table[col sep=comma, x=theta, y=difference]
                                {hardy/plot6_unit_circle_modulus_display.csv};
                                \addlegendentry{\footnotesize{$||B_n(e^{i\theta})|-|B_n^{\mathrm{SGD}}(e^{i\theta})||$}}
                            
                            \node[anchor=south west, font=\small] at (rel axis cs:0,1.02) {$10^{-12}+1$};
                            \end{axis}
                        \end{tikzpicture}                       
                    \caption{Difference of boundary modulus on the unit circle.}
                    \label{fig:hardy_unit_circle_modulus}
                \end{subfigure}
                \caption{Comparison of the exact Blaschke product and the reconstruction obtained from the recovered Hardy-space kernel coefficients on the unit circle $z=e^{i\theta}$.}
                \label{fig:hardy_unit_circle_comparison}
            \end{figure}
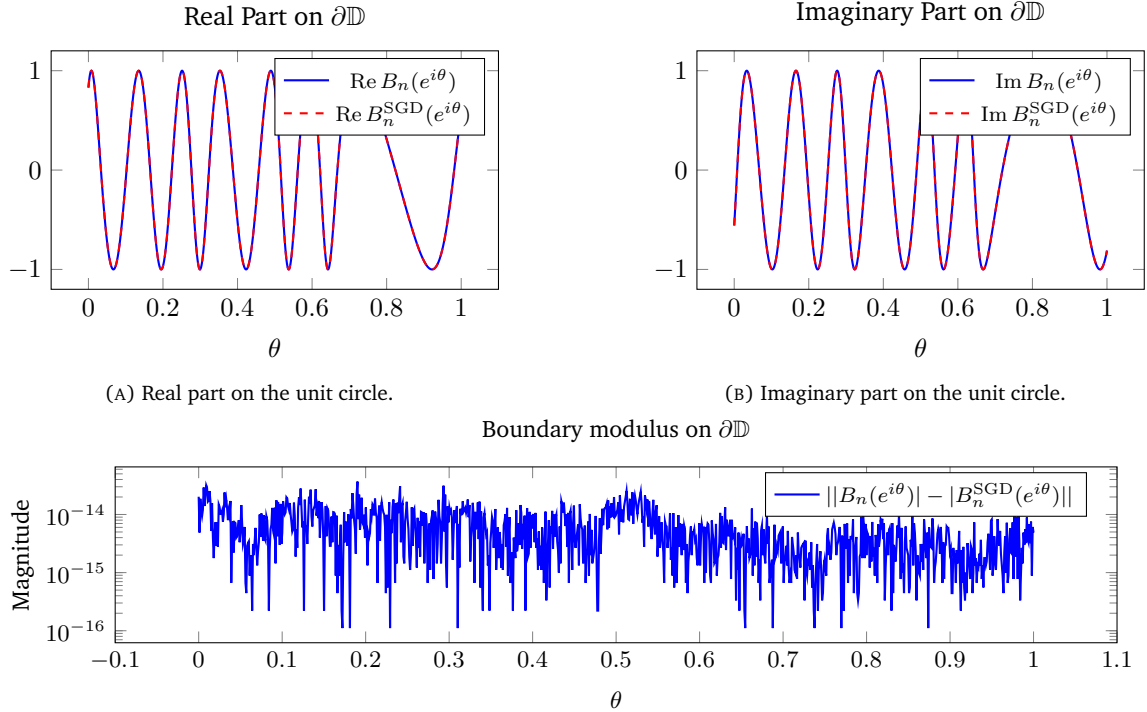
            
            In conclusion, the complex-SGD gives a unified and flexible framework for approximating target functions, including generalized supershifts. In particular, it makes it possible to explore other examples that previously could only be obtained through more direct constructions, as in \cite{AlpayDeMartinoDiki2025ComplexRepresenter}. We expect that this perspective will also be useful in other reproducing kernel spaces. In particular, the Mittag-Leffler--Fock space appears to be a natural setting for future work.

\section{Acknowledgements}
This material is based upon work supported by the National Science Foundation under Grant No. DMS-1929284 while EB was in residence at the Institute for Computational and Experimental Research in Mathematics in Providence, RI, during the "Stochastic and Randomized Algorithms in Scientific Computing: Foundations and Applications" semester program.

\newpage
\appendix

\section{Reproducing Kernel Hilbert Space Background}
\label{sec:rkhs}
    A reproducing kernel Hilbert space (RKHS) is a Hilbert space $\mcal H$ of real or complex-valued functions $f:\Omega\ra \C$ such that point-evaluation functionals are continuous (or equivalently, bounded). Such a space admits a function $K:\Omega\times \Omega \ra \C$, called the reproducing kernel, that has the following properties:
    \begin{enumerate}[(1)]
        \item For all $w\in \Omega$, we have $K_w(\cdot)\Def K(\cdot, w)\in \mcal H$.
        \item For all $f\in \mcal F$, $\lll f,K_w\rrr=f(w)$.
    \end{enumerate}
    We note that $K$ is uniquely defined by the Riesz representation theorem. Another key property of $K$ is that $K$ is positive semi-definite, i.e. for all $n\in \N$, $w_1,\dots, w_n\in \Omega$ and $c_1,\dots, c_n\in \C$, we have
    \[
        \sum_{i,j=1}^n c_j\cl{c_i}K(w_i,w_j)\geq 0.
    \]
    It can also be shown the matrix $K\Def \{K(w_i,w_j)\}_{1\leq   i,j\leq   n}$ is Hermitian. For more details specific to RKHS's, we refer the reader to \cite{saitoh1988}.
    
    An important RKHS result leveraged in applications is the ability to reduce the task of minimizing an objective function over an infinite dimensional space $\mcal H$ to a finite one. To see why this is possible, we first present the representer theorem, which was originally devised for real Hilbert spaces and then extended recently to complex Hilbert spaces \cite{AlpayDeMartinoDiki2025ComplexRepresenter}:
    
    \begin{theorem}[Complex Representer Theorem]
        \label{thm:crt}
        Let
        ${\mathcal{H}}$ be an RKHS with kernel $K$ and norm $\|\cdot\|_{{\mathcal{H}}}$. Suppose $g$
        a strictly monotonically increasing real-valued function on $[0, \infty)$.
        Let us consider the functional  $J:{\mathcal{H}}\to \mathbb{R}$ defined by
        \begin{equation}
            \label{eq:J}
            F(f) = L_{\bm{w}}\left(\left(w_1, f\left(z_1\right)\right), \ldots,\left( w_n, \left(z_n\right)\right)\right) + g(\|f\|_{{\mathcal{H}}}).
        \end{equation}
        for $f\in {\mathcal{H}}$, where $L_{\bm{w}}:\mathbb{C}^n \rightarrow\mathbb{R} \cup\{\infty\}$ is an arbitrary loss function for the dataset $\{(z_j,w_j)\}_{j=1}^n$. Then, any solution $f_\ast$ to the optimization problem
        \[
            f_\ast\in \underset{f\in{\mathcal{H}}}\argmin ~  F(f),
        \]
        has the form
        \begin{equation}
            \label{eq:fmin}
            f_\ast(\cdot)=\sum_{k=j}^n \alpha_j K\left(\cdot, z_j\right) .
        \end{equation}
        for suitable coefficients $\alpha_1,\dots,\alpha_n\in\mathbb{C}$.
    \end{theorem}
    
    It stands to reason that instead of minimizing over an infinite-dimensional space $\mcal H$, it would be more feasible to minimize over the finite-dimensional space of coefficients $\balpha =(\alpha_1,\dots,\alpha_n)^T\in \C^n$. This reduction is especially convenient in the least-squares setting.
    
    Let $\{(z_i,y_i)\}_{i=1}^n$ be a given data set, where $z_i$ are sampled points and $y_i$ are the corresponding observations. We consider the regularized least-squares functional, $F$, over a given RKHS $\mathcal{H}$ with kernel denoted by $K$:
    \begin{equation}
        F(f)=\frac{1}{2}\sum_{i=1}^n |y_i-f(z_i)|^2+\lambda\|f\|_{\mathcal H}^2, \qquad \lambda>0. \tag{\ref{eq:kernel-LS-functional}}
    \end{equation}
    By the Representer Theorem, the minimizer of \eqref{eq:kernel-LS-functional} is of the form
    \begin{equation}\label{eq:kernel-representer-LS}
        f_*(z)=\sum_{j=1}^n \alpha_j K(z,z_j).
    \end{equation}
    
    Letting $\by=(y_1,\dots,y_n)^T\in\mathbb{C}^n$ and $K\in\mathbb{C}^{n\times n}$ denote the Gram matrix
    \[
        K_{ij}=K(z_i,z_j), \qquad 1\leq   i,j\leq   n,
    \]
    we see the vector of fitted values at the sample points is
    \[
        \bigl(f_\ast(z_1),\dots,f_\ast(z_n)\bigr)^T=K\balpha.
    \]
    Moreover, the norm of $f_\ast$ satisfies
    \[
        \|f_\ast\|_{\mcal H}^2=\left\|\sum_{j=1}^n \alpha_j K(\cdot,z_j)\right\|_{\mathcal H}^2=\sum_{i,j=1}^n \overline{\alpha_i}\alpha_j K(z_i,z_j)=\balpha^*K\balpha.
    \]
    Hence the functional \eqref{eq:kernel-LS-functional} can be rewritten as
    \begin{equation}
        F(\balpha)=\frac{1}{2}\|\by-K\balpha\|^2+\lambda\,\balpha^*K\balpha=\frac{1}{2}(\by-K\balpha)^\ast(\by-K\balpha)+\lambda\,\balpha^*K\balpha\tag{\ref{eq:Falpha-kernel}}.
    \end{equation}
    Thus, the minimization problem in $\mathcal H$ is reduced to a finite-dimensional optimization problem of recovering a coefficient vector $\balpha\in\C^n$. However, RKHS's admit another reduction simplifying the task of recovering $\balpha_\ast$.
    
    \begin{proposition}
    [{\cite[Proposition 2.5]{AlpayDeMartinoDiki2025ComplexRepresenter}}]
        \label{prop:K_plus_lamba_I}
        Let $\y = (y_0, \ldots, y_n)^T\in \mathbb{C}^{(n+1) \times 1}$, $\balpha= (\alpha_0, \ldots, \alpha_n)^T\in \mathbb{C}^{(n+1) \times 1}$ and $K \in \mathbb{C}^{(n+1) \times (n+1)}$ being a symmetric matrix with entries $K(z_i, z_j)$ for $0 \leq   i, j \leq   n$. Then $ F(\balpha)$ (as given in \eqref{eq:Falpha-kernel}) is minimized 
        if and only if
        \begin{equation}
            \label{eq:w}
            \y = (K+\lambda I)\balpha.
        \end{equation}
    \end{proposition}
    
    This allows us to minimize 
    \[
        \frac{1}{2}\|(K+\lambda I)\balpha\|^2,    
    \]
    instead of the regularized problem in \Cref{eq:Falpha-kernel} in order to recover the function $f_\ast$. This is useful because complex-SGD admits a simpler gradient update when there is no regularization term; the exact details are presented in \Cref{sec:complex_grad_computation}. 

\newpage
\section{SGD Proofs}
\label{sec:proofs}
    \subsection{Convergence of Average of SGD Iterates}
    \label{sec:proof_conv_avg}
        With some abuse of notation, we will write $f(\z)$ and $\nabla f(\z)$ to mean the random vectors $f(\z;\omega)$ and $\nabla f(\z;\omega)$, respectively. We also note that the following proofs are originally from \cite{handbook} but are adapted to fit the new complex-parameter scheme with the assumptions outlined in \Cref{sec:problem_setup}.
        \begin{lemma}
            With the assumptions of \Cref{sec:problem_setup}, except \Cref{assum:mu_convex_complex}, we have
            \[
                \frac{1}{2L}\E\|\nabla f(\z)-\nabla f(\zstar)\|^2\leq   F(\z)-F(\zstar).
            \]
        \end{lemma}
        \begin{proof}
            Behold:
            \begin{align*}
                f(\x)-f(\y) &= f(\x)-f(\z) + f(\z)-f(\y) \\
                &\leq   \lll \nabla f(\x), \x-\z \rrr_\R + \lll \nabla f(\y),\z-
                \y\rrr_\R + \frac{L}{2}\|\z-\y\|^2.
            \end{align*}
            In the ordinary case, we use $\z=\y-\frac{1}{L}(\nabla f(\y)-\nabla f(\x))$ to minimize the last term. For complex SGD, this will suffice as well. Doing such yields
            \begin{align*}
                f(\x)-f(\y)&\leq   \lll \nabla f(\x),\x-\y\rrr_\R + \frac{1}{L}\lll \nabla f(\x),\nabla f(\y)-\nabla f(\x)\rrr_\R - \frac{1}{L}\lll \nabla f(\y),\nabla f(\y)-\nabla f(\x)\rrr_\R + \frac{1}{2L}\|\nabla f(\y)-\nabla f(\x)\|^2 \\
                &= \lll \nabla f(\x),\x-\y\rrr_\R - \frac{1}{L}\lll \nabla f(\y)-\nabla f(\x),\nabla f(\y)-\nabla f(\x)\rrr_\R + \frac{1}{2L}\|\nabla f(\y)-\nabla f(\x)\|^2 \\
                &=\lll \nabla f(\x),\x-\y\rrr_\R - \frac{1}{2L}\|\nabla f(\y)-\nabla f(\x)\|^2.
            \end{align*}
            By re-arranging we have
            \[
                \frac{1}{2L}\|\nabla f(\y)-\nabla f(\x)\|^2\leq   \lll \nabla f(\x),\x-\y\rrr_\R + f(\y)-f(\x),
            \]
            so
            \begin{align*}
                \frac{1}{2L}\E\|\nabla f(\z)-\nabla f(\zstar)\|^2 &\leq   \E[ (\lll \nabla f(\zstar),\zstar-\z\rrr_R] + \E f(\z)-\E f(\zstar) \\
                &=\lll \nabla F(\zstar),\zstar-\z\rrr_\R + F(\z)-F(\zstar) \\
                &= F(\z)-F(\zstar).
            \end{align*}
        \end{proof}
        
        This leaves us equipped to prove \Cref{thm:avg_iter_converge_complex}:
        
        \begin{proof}[Proof of
        \Cref{thm:avg_iter_converge_complex}]
        Updates are given by $\z^{(t+1)}=\z^{(t)}- \eta_t \nabla f(\z^{(t)};\omega)$, so if we let $\zstar$ be the minimizer of $F$, we may write
        \[
            \|\z^{(t+1)}-\zstar\|^2=\|\z^{(t)}-\zstar\|^2+2\eta_t \lll \nabla f(\z^{(t)}),\zstar - \z^{(t)}\rrr_\R + \eta_t^2\|\nabla f(\z^{(t)})\|^2.
        \]
        By \Cref{assum:convex_complex}, we have
        \[
            2\eta_t \lll \nabla F(\z^{(t)}),\zstar - \z^{(t)}\rrr_\R \leq   2\eta_t (F(\zstar)-F(\z^{(t)})),
        \]
        so with \Cref{assum:unbiased_complex}
        \begin{align*}
            \E\groupb{\|\z^{(t+1)}-\zstar\|^2\mid \z^{(t)}} &= \|\z^{(t)}-\zstar\|^2+2\eta_t \lll \nabla F(\z^{(t)}),\zstar - \z^{(t)}\rrr_\R + \eta_t^2 \E\groupb{\|\nabla f(\z^{(t)})\|^2\mid \z^{(t)}} \\
            &\leq   \|\z^{(t)}-\zstar\|^2+2\eta_t(F(\z^{(t)})-F(\zstar)) + \eta_t^2 \E\groupb{\|\nabla f(\z^{(t)})\|^2\mid \z^{(t)}}.
        \end{align*}
        To bound the last term, we first observe
        \begin{align*}
            \E\groupb{\|\nabla f(\z^{(t)})\|^2\mid \z^{(t)}} &\leq   2\E\groupb{\|\nabla f(\z^{(t)})-\nabla f(\zstar)\|^2\mid \z^{(t)}}+2\E\groupb{\|\nabla f(\zstar)\|^2\mid \z^{(t)}} \\
            &\leq   4L(F(\z^{(t)})-F(\zstar))+2\s_\ast.
        \end{align*}
        \Cref{assum:bounded_var_complex} and \Cref{assum:stationary_minima_complex} ensure $\s_\ast<\infty$. Then we have
        \begin{align*}
            \E\groupb{\|\z^{(t+1)}-\zstar\|^2\mid \z^{(t)}} &\leq   \|\z^{(t)}-\zstar\|^2+2\eta_t(2\eta_t L - 1) (F(\z^{(t)})-F(\zstar))+2\eta_t^2 \s_\ast \\
            &\leq   \|\z^{(t)}-\zstar\|^2-\eta_t (F(\z^{(t)})-F(\zstar))+2\eta_t^2 \s_\ast.
        \end{align*}
        because $\eta_t<\frac{1}{4L}$, i.e. $2(2\eta_t L-1)\leq   -1$. We may re-arrange this and take the total expectation to see
        \[
            \eta_t \E[F(\z^{(t)})-F(\zstar)]\leq   \E\|\z^{(t)}-\zstar\|^2 -\E\|\z^{(t+1)}-\zstar\|^2+2\eta_t^2\s_\ast,
        \]
        whereby
        \[
            \sum_{t=0}^{T-1} \eta_t \E[F(\z^{(t)})-F(\zstar)] \leq   \|\z^{(0)}-\zstar\|^2 + 2\s_\ast \sum_{t=0}^{T-1} \eta_t^2.
        \]
        Finally, we see
        \begin{align*}
            \E\groupb{F\groupp{\sum_{t=0}^{T-1}\frac{\eta_t}{\sum_{i=0}^{T-1}\eta_i}\z^{(t)}}-F(\zstar)}&\leq   \E\groupb{\sum_{t=0}^{T-1}\frac{\eta_t}{\sum_{i=0}^{T-1}\eta_i}(F(\z^{(t)})-F(\zstar))} \\
            &\leq   \frac{\|\z^{(0)}-\zstar\|^2}{\sum_{t=0}^{T-1} \eta_t}+\frac{2\s_\ast \sum_{t=0}^{T-1}\eta_t^2}{\sum_{t=0}^{T-1} \eta_t}.
        \end{align*}
        \end{proof}
    
    \subsection{Convergence in \texorpdfstring{$\mu$}{mu}-Strongly Convex Case}\label{sec:proof_mu_strong_conv}
        As in \Cref{sec:proof_conv_avg}, we will write $f(\z)$ and $\nabla f(\z)$ to mean the random vectors $f(\z;\omega)$ and $\nabla f(\z;\omega)$, respectively.
        \begin{proof}[Proof of \Cref{thm:mu_convex_convergence_complex}]
            Similarly to the proof of \Cref{thm:avg_iter_converge_complex}, we will use the decomposition
            \[
                \|\z^{(t+1)}-\zstar\|^2=\|\z^{(t)}-\zstar\|^2-2\eta_t \lll \nabla f(\z^{(t)}), \z^{(t)}-\zstar \rrr_\R + \eta_t^2\|\nabla f(\z^{(t)})\|^2.
            \]
            Then we have
            \begin{align*}
                \E \groupb{\|\z^{(t+1)}-\zstar\|^2\mid \z^{(t)}}&= \|\z^{(t)}-\zstar\|^2-2\eta_t \lll \nabla F(\z^{(t)}),\z^{(t)}-\zstar\rrr_\R +\eta_t^2\E\groupb{\|\nabla f(\z^{(t)})\|^2\mid \z^{(t)}}.
            \end{align*}
            Taking total expectation after using \Cref{assum:mu_convex_complex} and the inequality
            \[
                \E\groupb{\|\nabla f(\z^{(t)})\|^2\mid \z^{(t)}} \leq   4L(F(\z^{(t)})-F(\zstar))+2\s_\ast.
            \]
            from the proof of \Cref{thm:avg_iter_converge_complex}, we have
            \begin{align*}
                \E\|\z^{(t+1)}-\zstar \|^2 &\leq   (1-\eta_t\mu)\E\|\z^{(t)}--zstar\|^2+2\eta_t^2\s_\ast + 2\eta_t(2\eta_t L - 1)(F(\z^{(t)})-F(\zstar))\\
                &\leq   (1-\eta_t\mu)\E\|\z^{(t)}-\zstar\|^2+2\eta_t^2\s_\ast.  \tag{$\eta_t < \frac{1}{2L}$}
            \end{align*}
            By repeating this inequality for lesser $t$ and assuming $\eta_t\geq c>0$ for all $t$, we arrive at
            \begin{align*}
                \E\|\z^{(t+1)}-\zstar\|^2 &\leq   \groupb{\prod_{j=0}^t (1-\eta_j\mu)}\|\z^{(0)}-\zstar\|^2+2\s_\ast \sum_{j=0}^t \eta_{t-j}^2(1-\eta_i \mu)^j\\
                &\leq   (1-c\mu)^{t+1}\|\z^{(0)}-\zstar\|^2+\frac{\s_\ast}{2L^2} \sum_{j=0}^\infty (1-c \mu)^j \\
                &\leq    (1-c\mu)^{t+1}\|\z^{(0)}-\zstar\|^2+\frac{\s_\ast}{2L^2c\mu}.
            \end{align*}
            If instead the step sizes are constant, i.e. $\eta_t=\eta \in (0,1/2L)$ for all $t$, we may replace the second inequality with
            \[
                (1-\eta \mu)^{t+1}\|\z^{(0)}-\zstar\|^2+\frac{2\s_\ast \eta}{\mu}.
            \]
        \end{proof}
    
    \subsection{Stationary convergence of complex SGD with adaptive stepsizes}
        \begin{proof}[Proof of \Cref{thm:stationary_adaptive_complex}]
            Let $g_t:=\nabla f(\z^{(t)};\omega_t)$, so $\z^{(t+1)}=\z^{(t)}-\eta_t g_t$.
            By \Cref{assum:L_lip_complex} applied to $F$ with $\w=\z^{(t+1)}$ and $\z=\z^{(t)}$,
            \[
                F(\z^{(t+1)}) \leq F(\z^{(t)})+\lll \nabla F(\z^{(t)}),\z^{(t+1)}-\z^{(t)}\rrr_{\R}+\frac{L}{2}\|\z^{(t+1)}-\z^{(t)}\|^2.
            \]
            Substituting $\z^{(t+1)}-\z^{(t)}=-\eta_t g_t$ gives
            \[
                F(\z^{(t+1)})\leq F(\z^{(t)})-\eta_t\lll \nabla F(\z^{(t)}),g_t\rrr_{\R}+\frac{L\eta_t^2}{2}\|g_t\|^2.
            \]
            Taking conditional expectation given $\z^{(t)}$ and using \Cref{assum:unbiased_complex},
            \[
                \E\big[\lll \nabla F(\z^{(t)}),g_t\rrr_{\R}\,\big|\,\z^{(t)}\big]=\lll \nabla F(\z^{(t)}),\E[g_t\mid \z^{(t)}]\rrr_{\R}=\|\nabla F(\z^{(t)})\|^2.
            \]
            Moreover, by the second-moment decomposition and \Cref{assum:bounded_var_complex},
            \[
                \E[\|g_t\|^2\mid \z^{(t)}]=\|\nabla F(\z^{(t)})\|^2+\E[\|g_t-\nabla F(\z^{(t)})\|^2\mid \z^{(t)}]\leq   \|\nabla F(\z^{(t)})\|^2+\sigma^2.
            \]
            Combining yields
            \[
                \E[F(\z^{(t+1)})\mid \z^{(t)}]\leq F(\z^{(t)})-\eta_t\Big(1-\frac{L\eta_t}{2}\Big)\|\nabla F(\z^{(t)})\|^2+\frac{L\eta_t^2}{2}\sigma^2.
            \]
            If $\eta_t\leq 1/L$, then $1-\frac{L\eta_t}{2}\geq \frac12$, hence
            \[
                \frac{\eta_t}{2}\|\nabla F(\z^{(t)})\|^2\leq F(\z^{(t)})-\E[F(\z^{(t+1)})\mid \z^{(t)}]+\frac{L\eta_t^2}{2}\sigma^2.
            \]
            Taking total expectations and summing from $t=0$ to $T-1$ yields a telescoping sum:
            \[
                \frac12\sum_{t=0}^{T-1}\eta_t\,\E\|\nabla F(\z^{(t)})\|^2\leq F(\z^{(0)})-\E[F(\z^{(T)})]+\frac{L\sigma^2}{2}\sum_{t=0}^{T-1}\eta_t^2 \leq F(\z^{(0)})-F_\star+\frac{L\sigma^2}{2}\sum_{t=0}^{T-1}\eta_t^2.
            \]
            Multiplying by $2$ gives the claimed bound, and the minimum bound follows since $\min_t a_t\leq (\sum_t \eta_t a_t)/(\sum_t \eta_t)$ for $\eta_t>0$.
        \end{proof}

\newpage

\section{Directional Bias Proof}
\label{sec:directional_bias_proof}
    For completeness, we include the proof of \Cref{cor:dir_bias}, but we also note that it is due to Steinerberger \cite{RK_small_sing} and only minor modifications are made to hold in the complex SGD and inconsistent settings.
    \begin{proof}[Proof of \Cref{cor:dir_bias}]
        Firstly, we assume the system $A\x=\b$ is consistent with solution $\zstar$. With this assumption, we may assume $\b=0$ as was done in \cite{RK_small_sing}.
    
        With the simplifying assumption $\b=0$, we have $\z^{(t+1)}=\z^{(t)}-\eta n \a_i \z^{(t)}\a_i\dl$. If we let $\bm{v}_k$ denote the $k$-th right singular vector of $A$, we may write
        \begin{align*}
            \E \lll \z^{(t+1)},\bm{v}_k\rrr &= \sum_{i=1}^m \frac{1}{m}\lll \z^{(t)}-\eta n \a_i \z^{(t)}\a_i\dl,\bm{v}_k\rrr \\
            &=\frac{1}{m}\sum_{i=1}^m(\lll \z^{(t)},\bm{v}_k\rrr-\eta n \a_i \z^{(t)}\lll \a_i\dl, \bm{v}_k\rrr)\\
            &=\lll \z^{(t)},\bm{v}_k\rrr - \frac{\eta n}{m} \sum_{i=1}^m \a_i\z^{(t)}\lll \a_i\dl,\bm{v}_k\rrr.
        \end{align*}
        Because we are using the complex inner product, we note that $\lll \a_i\dl,\bm{v}_k\rrr$ is the conjugate of the $i$-th entry of $A\bm{v}_k$. It is also clear $\a_i\z^{(t)}$ is the $i$-th entry of $A\z^{(t)}$. Then we may write
        \[
            \sum_{i=1}^m \a_i\z^{(t)}\lll \a_i\dl,\bm{v}_k\rrr = \lll A\z^{(t)}, A\bm{v}_k\rrr.
        \]
        If we let $\bm{u}_k$ denote the $k$-th right singular vector of $A$ satisfying $A\bm{v}_k=\s_k \bm{u}_k$, we may observe
        \begin{align*}
            \lll \z^{(t)},\bm{v}_k\rrr - \frac{\eta n}{m} \sum_{i=1}^m \a_i\z^{(t)}\lll \a_i\dl,\bm{v}_k\rrr &= \lll \z^{(t)},\bm{v}_k\rrr - \frac{\eta n}{m} \lll A\z^{(t)}, A\bm{v}_k\rrr \\
            &= \lll\z^{(t)},\bm{v}_k\rrr - \frac{\eta n}{m} \left\lll A\groupp{\sum_{i=1}^n \lll \z^{(t)},\bm{v}_i\rrr \bm{v}_i}, A\bm{v}_k\right\rrr \\
            &=\lll\z^{(t)},\bm{v}_k\rrr - \frac{\eta n}{m} \left\lll \sum_{i=1}^n \s_i\lll \z^{(t)},\bm{v}_i\rrr \bm{u}_i, \s_k\bm{u}_k\right\rrr \\
            &= \lll\z^{(t)},\bm{v}_k\rrr - \frac{\eta n}{m} \s_k \lll\z^{(t)},\bm{v}_k\rrr \\
            &= \lll\z^{(t)},\bm{v}_k\rrr \groupp{1-\frac{\eta n \s_k^2}{m}}.
        \end{align*}
        By taking the total expectation and iterating, we have
        \[
            \E\lll \z^{(t+1)},\bm{v}_k\rrr = \groupp{1-\frac{\eta n \s_k^2}{m}}^{t+1}\lll \z^{(0)},\bm{v}_k\rrr.
        \]
    
        Now we consider the inconsistent case. To model this, we will apply SGD to the system $A\x=\b+\bm{\ve}$, which we assume is inconsistent, but we also assume $\zstar$ is a solution to the system $A\x=\b$. Then we may write
        \begin{align*}
            \E \lll \z^{(t+1)}-\zstar,\bm{v}_k\rrr &= \sum_{i=1}^m \frac{1}{m}\lll \z^{(t)}+\eta n (b_i+\ve_i - \a_i \z^{(t)})\a_i\dl-\zstar,\bm{v}_k\rrr \\
            &=\frac{1}{m}\sum_{i=1}^m(\lll \z^{(t)}-\zstar,\bm{v}_k\rrr-\eta n \a_i (\z^{(t)}-\zstar)\lll \a_i\dl, \bm{v}_k\rrr)+\frac{\eta n}{m}\sum_{i=1}^m \ve_i \lll \a_i\dl,\bm{v_k}\rrr \\
            &=\lll \z^{(t)}-\zstar,\bm{v}_k\rrr - \frac{\eta n}{m} \sum_{i=1}^m \a_i(\z^{(t)}-\zstar)\lll \a_i\dl,\bm{v}_k\rrr + \frac{\eta n}{m}\sum_{i=1}^m \ve_i \lll \a_i\dl,\bm{v_k}\rrr .
        \end{align*}
        We may reproduce the same steps as in the consistent setting to conclude
        \[
            \lll \z^{(t)}-\zstar,\bm{v}_k\rrr - \frac{\eta n}{m} \sum_{i=1}^m \a_i(\z^{(t)}-\zstar)\lll \a_i\dl,\bm{v}_k\rrr = \lll \z^{(t)}-\zstar ,\bm{v}_k\rrr \groupp{1-\frac{\eta n \s_k^2}{m}},
        \]
        and
        \[
            \frac{\eta n}{m}\sum_{i=1}^m \ve_i \lll \a_i\dl,\bm{v_k}\rrr = \frac{\eta n}{m}\lll \bm{\ve},A\bm{v}_k\rrr =  \frac{\eta n \s_k}{m}\lll \bm{\ve},\bm{u}_k\rrr.
        \]
        Then
        \[
            \E \lll \z^{(t+1)}-\zstar,\bm{v}_k\rrr = \lll \z^{(t)}-\zstar ,\bm{v}_k\rrr \groupp{1-\frac{\eta n \s_k^2}{m}} +  \frac{\eta n \s_k}{m}\lll \bm{\ve},\bm{u}_k\rrr \leq   \lll \z^{(t)}-\zstar ,\bm{v}_k\rrr \groupp{1-\frac{\eta n \s_k^2}{m}} + \frac{\eta n \s_k}{m}\|\bm{\ve}\|.
        \]
        By taking the total expectation and iterating, we have
        \begin{align*}
            \E \lll \z^{(t+1)}-\zstar,\bm{v}_k\rrr &= \groupp{1-\frac{\eta n \s_k^2}{m}}^{t+1}\lll \z^{(0)}-\zstar,\bm{v}_k\rrr + \sum_{i=0}^{t}\frac{\eta n \s_k}{m}\groupp{1-\frac{\eta n \s_k}{m}}^{t-i} \|\bm{\ve}\| \\
            &\leq   \groupp{1-\frac{\eta n \s_k^2}{m}}^{t+1}\lll \z^{(0)}-\zstar,\bm{v}_k\rrr+\|\bm{\ve}\|.
        \end{align*}
        We may similarly show
        \[
            \E \lll \z^{(t+1)}-\zstar,\bm{v}_k\rrr \geq \groupp{1-\frac{\eta n \s_k^2}{m}}^{t+1}\lll \z^{(0)}-\zstar,\bm{v}_k\rrr-\|\bm{\ve}\|.
        \]
    \end{proof}

\newpage

\section{Least Squares Objective with Complex Gradient}
\label{sec:least_squares_example}
    \subsection{Complex Gradient Computation}
    \label{sec:complex_grad_computation}
        The least squares objective for a system $A\x=\bm{b}$ for $A\in\C^{m\times n}$ may be expressed as
        \[
            F(\z)=\frac{1}{2}\|A\z-\b\|^2=\frac{m}{2}\sum_{j=1}^m \frac{1}{m}| \a_j\z - b_j|^2 = \E f_j(\z),
        \]
        where $\a_j\in \C^{1\times n}$ is the $j$-th row of $A$ and $f_j(\z)=\frac{m}{2}|\a_j\z-b_j|^2$ for $j\sim \unif(m)$. We will show
        \[
            \nabla f_j (\z)=2\nabla_{\cl{\z}}f_j(\z)=n(\a_j\z-b_j)\a_j\dl,
        \]
        whereby
        \begin{equation}
            \nabla F(\z)=\sum_{j=1}^m(\a_j\z-b_j)\a_j\dl=A\dl (A\z-\b). \label{eq:complex_grad_least_square}
        \end{equation}
        
        \begin{lem}
            \label{lem:comp_sgd_update_LS}
            Using the (complex) gradient, we have $\nabla f_j(\z)=n(\a_j\z-b_j)\a_j\dl$.
            \end{lem}
        \begin{proof}
        First we write the entries of the residual $\rvec$ as:
        \[
            r_j(\z):=\a_j \z-b_j=\sum_{k=1}^m a_{j,k}z_k-b_j.
        \]
        Then, writing $a_{j,k}=\Re(a_{j,k})+i\Im(a_{j,k})$ and letting $z_k=x_k+iy_k$ with $x_k,y_k\in \R$, we have
        \[
            \Re(r_j)=-\Re(b_j)+\sum_{k=1}^m\Big(\Re(a_{j,k})x_k-\Im(a_{j,k})y_k\Big),
        \]
        \[
            \Im(r_j)=-\Im(b_j)+\sum_{k=1}^m\Big(\Re(a_{j,k})y_k+\Im(a_{j,k})x_k\Big).
        \]
        Using $|r_j|^2=\Re(r_j)^2+\Im(r_j)^2$, we can write
        \[
            f_j(\z)=\frac{m}{2}\Big(\Re(r_j)^2+\Im(r_j)^2\Big).
        \]
        Differentiating with respect to $x_i$ and $y_i$ gives
        \begin{align}
            \frac{\partial}{\partial x_i}f_j(\z)
            &=n\Big(\Re(r_j)\Re(a_{j,i})+\Im(r_j)\Im(a_{j,i})\Big), \label{eq:app_dx}\\
            \frac{\partial}{\partial y_i}f_j(\z)
            &=n\Big(\Re(r_j)(-\Im(a_{j,i}))+\Im(r_j)\Re(a_{j,i})\Big). \label{eq:app_dy}
        \end{align}
        Therefore,
        \begin{align*}
            \frac{\partial}{\partial \bar z_i}f_j(\z)
            &=\frac12\Big(\frac{\partial}{\partial x_i}f_j(z)+i\frac{\partial}{\partial y_i}f_j(z)\Big)\\
            &=\frac{n}{2}\Big(
            \Re(r_j)\Re(a_{j,i})+\Im(r_j)\Im(a_{j,i})
            +i\big(\Re(r_j)(-\Im(a_{j,i}))+\Im(r_j)\Re(a_{j,i})\big)
            \Big)\\
            &=\frac{n}{2}\,r_j(\z)\,\overline{a_{j,i}}.
        \end{align*}
        Now we may write the (complex) gradient of $f_j$ compactly as
        \[
            \nabla f_j(\z)=2\nabla_{\cl{\z}}f_j(\z)=n r_j(\z) \a_j\dl=n(\a_j \z-b_j)\a_j\dl.
        \]
        \end{proof}
        
        Suppose instead we use a regularized least squares objective when $A$ is square, which we may express as
        \[
            F(\z)= \frac{1}{2}\|A\z-\b\|^2 + \lambda \z\dl W \z,
        \]
        where $W\in \C^{n\times n}$ is positive semi-definite. Then we may use $f_j(\z)=\frac{m}{2}|\a_j\z-b_j|^2+\lambda \z\dl W \z$ to ensure $F(\z)=\E f_j(\z)$. Then we have
        \[
            \nabla f_j(\z)=n(\a_j\z-b_j) + 2\lambda W \z.
        \]
        To see why the (complex) gradient of $\z\dl W \z$ is $2W\z$, we may use the Hermitian square root of $W$, which we will denote $\tilde W$, to write
        \[
            \z\dl W \z = \z\dl \tilde{W} \tilde{W} \z = (\tilde W \z)\dl \tilde W \z = \|\tilde W z\|^2.
        \]
        Then using \Cref{eq:complex_grad_least_square}, we have $\nabla (\z\dl W\z)= 2\tilde W\dl (\tilde W \z)=2W\z$, yielding
        \[
            \nabla F(\z)=A\dl (A\z-\b) + 2 \lambda \tilde W\dl \tilde W \z = A\dl (A\z-\b)\a_j\dl + 2 \lambda W \z.
        \]
        In the special case where $A=W$, e.g. in regularized kernel regression, we have
        \[
            \nabla f_j(\z) = n(\a_j\z -b_j)+2\lambda W\z,\quad\quad \nabla F(\z) = A(A\z-\b)+2\lambda A\z.
        \]
    
    \subsection{Confirmation of Validity of Assumptions}
    \label{sec:validity_of_assumptions}
        We will show that the collection of assumptions presented \Cref{sec:problem_setup} is valid by showing the least-squares objective
        \[
            F(\z) = \frac{1}{2}\|A\z-\b\|^2.
        \]
        satisfies them. Here we assume $A$ has full rank and is overdetermined. We note that the least-squares solution $\zstar$ is given by $(A\dl A)\inv A\dl \b$, so
        \[
            \nabla F(\zstar) = A\dl(A\z-\b)=A\dl b-A\dl b=0,
        \]
        whereby \Cref{assum:stationary_minima_complex} is satisfied. For simplicity, instead of showing \Cref{assum:bounded_var_complex}, we just note that $\E\|\nabla f_j(\zstar)\|^2<\infty$ because this is all that is required for \Cref{thm:avg_iter_converge_complex} and \Cref{thm:mu_convex_convergence_complex}. However, this fact is obvious because each $\|\nabla f_j(\zstar)\|^2<\infty$ so $\E\|\nabla f_j(\zstar)\|^2=\sum_{j=1}^m \frac{1}{m}\|\nabla f_j(\zstar)\|^2<\infty$.
        
        Firstly, we note that
        \[
            F(\z)=\frac{m}{2}\sum_{j=1}^m \frac{1}{m}| \a_j\z - b_j|^2 = \E f_j(\z),
        \]
        where $\a_j$ is the $j$-th row of $A$ and $f_j(\z)=\frac{m}{2}|\a_j\z-b_j|^2$. Then it is clear \Cref{assum:bounded_below_complex} is satisfied. Additionally, because the complex gradient is real-linear and $i\sim \mcal D$ follows a discrete distribution, we automatically satisfy \Cref{assum:unbiased_complex}.
        
        We will show that \Cref{assum:L_lip_complex} holds for $F$, and showing it holds for each $f_j$ follows almost identically. In other words, we wish to show
        \[
            \frac{1}{2}\|Aw-b\|^2\leq   \frac{1}{2}\|Az-b\|^2 + \lll\nabla F(z),w-z\rrr_\R + \frac{L}{2}\|w-z\|^2
        \]
        for some constant $L$. We note that with $L=\s_{\max}^2 = \|A\|^2$,
        \begin{align}
            &\|Az-b\|^2+2\lll A\dl (Az-b),w-z\rrr_\R + \s_{\max}^2\|w-z\|^2 \nonumber \\
            &\geq \|Az-b\|^2+2\Re(\lll Az-b, A(w-z)\rrr) +\|A(w-z)\|^2 \label{eq:conj_transpose_in_inner_prod}\\
            &= \|Az-b+A(w-z)\|^2 \label{eq:elem_lin_alg_identity}\\
            &=\|Aw-b\|^2. \nonumber
        \end{align}
        \Cref{eq:elem_lin_alg_identity} follows from the general identity
        \[
            \|z+w\|^2 = \|z\|^2+2\Re(\lll z,w\rrr) + \|w\|^2
        \]
        for any inner product $\lll \cdot , \cdot \rrr$ over $\C$ and its induced norm $\|\cdot\|$. \Cref{eq:conj_transpose_in_inner_prod} follows from the identity $\lll Az,w\rrr =\lll z,A\dl w\rrr$ for all $z\in \C^n$, $w\in \C^m$, and $A\in \C^{m\times n}$
        
        We immediately see \Cref{assum:mu_convex_complex} is satisfied by noting the inequality \Cref{eq:conj_transpose_in_inner_prod} is flipped when $\s_{\max}$ is replaced with $\s_{\min}$. \Cref{assum:convex_complex} follows similarly, except for our choice of $f_j$, we have $\s_{\min}=0$, so we cannot say each $f_j$ is $\mu$-strongly convex. Thankfully, this is not required.

\bibliographystyle{IEEEtran}
\bibliography{bibliography}
\end{document}

%% file: dependencies/dependencies.tex
\usepackage{lipsum}
\usepackage{amsfonts}
\usepackage{graphicx}
\usepackage{epstopdf}
\usepackage[lmargin=1in,rmargin=1in,tmargin=1in,bmargin=1in]{geometry}
\usepackage[section]{placeins}
\usepackage[font=small]{caption}
\usepackage[T1]{fontenc}
\usepackage{bbm,dsfont}
\usepackage{bm}
\usepackage{xcolor}

\usepackage{tcolorbox}
\usepackage{subcaption}
\usepackage{float}

\usepackage{
    amsmath,           
    mathtools,
    amssymb,           
    enumerate,		   
    graphicx,          
    lastpage,          
    multicol,          
    multirow,          
    diagbox,
    mathrsfs,
    thmtools,
    thm-restate,
    scrextend,
    algorithm,
    algpseudocode,
    subcaption,
    fancyhdr,
    bbm,
    tikz,
    pgfplots,
    pgfplotstable,
    textcomp,
    charter,
    ulem,
    changepage,
    stmaryrd,
    hyperref,
    xurl,
    etoolbox,
    cleveref
}
\usepackage[export]{adjustbox}
\usepackage{amsopn}
\usepgfplotslibrary{groupplots}

\newtheorem{theorem}{\textit{Theorem}}

\newtheorem*{definition}{\textit{Definition}}
\newtheorem*{lemma}{\textit{Lemma}}
\newtheorem{rem}{Remark}

\newtheorem{assum}{Assumption}

\newtheorem{proposition}[theorem]{Proposition}

\newenvironment{remark}{\begin{adjustwidth}{3em}{3em}\begin{rem}}{\\\end{rem}\end{adjustwidth}}

\newtheorem{lem}{\textit{Lemma}}

\newtheorem{cor}{\textit{Corollary}}

\newcommand{\groupp}[1]{\l(#1\r)}
\newcommand{\groupb}[1]{\l[#1\r]}

\newcommand{\groupabs}[1]{\l|#1\r|}

\newcommand{\mcal}[1]{\mathcal{#1}}
\renewcommand{\l}{\left}
\renewcommand{\r}{\right}

\renewcommand{\a}{\bm{a}}
\renewcommand{\b}{\bm{b}}
\newcommand{\E}{\mathbb{E}}
\newcommand{\N}{\mathbb{N}}

\newcommand{\R}{\mathbb{R}}
\newcommand{\rvec}{\bm{r}}
\newcommand{\C}{\mathbb{C}}
\newcommand{\x}{{\bm{x}}}
\newcommand{\y}{{\bm{y}}}
\newcommand{\z}{{\bm{z}}}
\newcommand{\zstar}{{\bm{z}_\ast}}
\newcommand{\w}{{\bm{w}}}

\newcommand{\s}{\sigma}

\newcommand{\ve}{\varepsilon}

\newcommand{\cl}[1]{\overline{#1}}
\newcommand{\inv}{^{-1}}
\newcommand{\dl}{^*}

\newcommand{\ra}{\rightarrow}
\renewcommand{\lll}{\langle}
\newcommand{\rrr}{\rangle}

\newcommand{\del}{\partial}

\renewcommand{\leq}{\leqslant}
\renewcommand{\geq}{\geqslant}
\newcommand{\Def}{\vcentcolon =}
\renewcommand{\Re}{\operatorname{Re}}
\renewcommand{\Im}{\operatorname{Im}}

\DeclareMathOperator*{\argmin}{arg\,min}

\DeclareMathOperator{\unif}{Unif}

\newcommand{\balpha}{{\bm{\alpha}}}
\newcommand{\bbeta}{{\bm{\beta}}}

\newcommand{\by}{{\bm{y}}}
\newcommand{\bw}{{\bm{w}}}